\documentclass[twoside,11pt]{article}

\usepackage{jmlr2e}

\usepackage{lastpage}
\usepackage{mathrsfs}
\usepackage{amsmath}
\usepackage{cleveref}
\usepackage{algorithm}
\usepackage{algorithmic}
\usepackage{booktabs}
\usepackage{placeins}
\hypersetup{hidelinks}
\crefname{figure}{Figure}{Figures}
\Crefname{figure}{Figure}{Figures}
\crefname{table}{Table}{Tables}
\Crefname{table}{Table}{Tables}
\crefname{section}{Section}{Sections}
\Crefname{section}{Section}{Sections}
\crefname{algorithm}{Algorithm}{Algorithms}
\Crefname{algorithm}{Algorithm}{Algorithms}
\jmlrheading{TBD}{2026}{1-\pageref{LastPage}}{June 30, 2026}{}{TBD}{Zhangyong Liang and Huanhuan Gao}

\ShortHeadings{Hybrid Iterative Neural Low-Regularity Integrator}{Liang and Gao}
\firstpageno{1}
\usepackage{tikz}
\usetikzlibrary{shapes, arrows.meta, positioning, fit, calc, backgrounds, shadows, decorations.pathreplacing, shapes.geometric}

\definecolor{boxbluedraw}{RGB}{58, 94, 120}      
\definecolor{boxbluefill}{RGB}{228, 236, 243}    
\definecolor{boxbrowndraw}{RGB}{186, 133, 81}    
\definecolor{boxbrownfill}{RGB}{253, 246, 240}   
\definecolor{textred}{RGB}{178, 34, 34}          
\definecolor{textorange}{RGB}{210, 105, 30}      

\begin{document}

\title{Hybrid Iterative Neural Low-Regularity Integrator for Nonlinear Dispersive Equations}

\author{\name Zhangyong Liang 
       \email zyliang1994@tju.edu.cn \\
       \addr National Center for Applied Mathematics \\
       Tianjin University \\
       Tianjin, 300072, China
       \AND \name Huanhuan Gao
       \email gao\_huanhuan@jlu.edu.cn\\
       \addr  School of Mechanical and Aerospace Engineering, Jilin University \\
       Jilin University \\
       Changchun, 130025, China
       }

\editor{Action Editor to be assigned}

\maketitle

\begin{abstract}
We propose HIN-LRI, a hybrid framework that augments a low-regularity integrator with a neural operator trained to correct its structured residual error.
A base low-regularity integrator provides the physical time-stepping backbone for nonlinear dispersive PDEs, while a lightweight neural corrector, operating on a low-dimensional latent manifold, approximates the residual defect that is difficult to control analytically.
The actual learned correction is denoted by $C_\theta(u,\tau):=\tau H_{\mathrm{neural}}(u,\tau;\theta)$, making the time-step scaling explicit in both the algorithm and the analysis.
Under a verifiable defect-approximation assumption, we prove a conditional stability and error-propagation result whose constants depend on the learned correction error rather than on an asserted exact cancellation.
The network is trained end-to-end through a solver-in-the-loop objective that unrolls the full iteration and penalizes trajectory error in a Bourgain-type norm, aligning learning with multi-step solver dynamics.
Experiments on three dispersive benchmarks with rough data show improved accuracy over analytical integrators, splitting methods, and neural PDE surrogates, together with spatial-refinement tests, out-of-distribution transfer, and online runtime measurements.
\end{abstract}


\section{Introduction}
\label{sec:introduction}

Nonlinear dispersive partial differential equations (PDEs), such as the Korteweg--de Vries (KdV) equation and the nonlinear Schr\"odinger (NLS) equation, play a fundamental role in describing a variety of physical phenomena, including shallow water waves, ion acoustic waves in plasmas, nonlinear optics, and Bose--Einstein condensates \citep{BaLlTi-2011-CPAM, KPV1993}. 
It is analytically established that these dispersive equations are globally well-posed in low-regularity Sobolev spaces $H^s$ (e.g., $s \ge -1$ for the KdV equation and $s \ge 0$ for the cubic NLS equation) \citep{Bourgain1993, KaTo2006, KiVi-2019}. 
In practical applications, however, the initial data may be intrinsically rough or highly oscillatory due to measurement noise, quantum fluctuations, or random background perturbations \citep{BD2009, Gubinelli-2012}. 
The development of robust computational methods for these equations with non-smooth solutions has historically faced severe challenges. 
Classical time discretizations, including finite difference methods, operator splitting methods, and traditional exponential integrators, rely on the boundedness of high-order time derivatives of the exact solution. This translates to requiring high spatial regularity \citep{splitting2, Ostermann-Su-2020}. 
When applied to rough data lacking sufficient smoothness, these classical schemes suffer from severe order reduction and spurious high-frequency numerical instabilities.
To bridge the gap between analytical well-posedness and numerical regularity requirements, low-regularity integrators (LRIs), also called resonance-based schemes, have been developed over the past decade \citep{Hofmanova-Schratz-2017, Ostermann-Schratz-FoCM, FengMaierhoferSchratz2024LongTime, Bronsard2024SymmetricNLS}.
Various semi-discrete and fully discrete schemes have since been proposed to mitigate the loss of derivatives in highly oscillatory regimes \citep{bronsard2022error, wang2022symmetric, WuZhao2022ELRIKdV, BrunedSchratz2022Trees}.
Further developments have successfully extended these techniques to incorporate structure-preserving properties and novel filtering techniques \citep{OstermannSchratz2018, LiWu2025UnfilteredKdV, banicamaierhoferschratz22, Knoller2019FourierNLS}.
The core philosophy of LRIs is to introduce twisted variables via the Lawson transformation, absorbing the stiff linear dispersive operator, and to exactly integrate the dominant high-frequency oscillatory phases in Fourier space.
Subsequent work introduced embedded LRIs for higher-order accuracy \citep{WuZhao-BIT} and unfiltered LRIs based on harmonic analysis \citep{Li-Wu-2021}. Discrete Bourgain space frameworks further extend convergence analyses toward lower Sobolev regularity \citep{RS-PAM-2022, ORS-JEMS, BrunedSchratz2022Trees}.

However, as the requirements for higher-order accuracy, multi-dimensional extensions, and rough-data robustness grow, purely analytical LRIs encounter significant analytical obstacles \citep{Oster2021ErrorNLS, RoussetSchratz2021Framework}. 
Attempting to resolve infinite-dimensional high-frequency oscillations using finite algebraic factorizations and local approximations leads to a regularity barrier. 
To construct higher-order schemes, one must evaluate complex multi-wave nested Duhamel integrals. 
Since these highly oscillatory phases cannot be integrated exactly in closed form, analytical LRIs are forced to apply local polynomial approximations or Taylor expansions \citep{luan2013exponential, hochbruck_ostermann_2010, shen2019geometric}. 
Expanding an exponential phase proportional to high-order spatial derivatives releases unbounded differential operators into the local truncation error. This reintroduces derivative loss and forces higher-order schemes to require greater regularity than the PDE itself demands. 
Furthermore, to maintain computational efficiency via the fast Fourier transform, analytical LRIs often rely on equation-specific resonance factorizations. 
This creates analytic rigidity. Minor physical perturbations can break the algebraic cancellations, and tracking higher-order interactions leads to rapidly growing combinatorial overhead \citep{BrunedSchratz2022Trees}. 
In addition, to avoid the severe order reduction caused by hard-truncation frequency filters, modern unfiltered LRIs attempt to preserve the full spectrum using averaging approximations \citep{LiWu2025UnfilteredKdV}. 
This generates a residual phase mismatch term whose present analytical control relies on logarithmically growing trilinear estimates. The resulting logarithmic factor appears in the global error bound and can make the algorithms vulnerable to nonlinear spectral aliasing. 
Finally, at endpoint or very low regularity, the lack of additional smallness in bilinear discrete Bourgain space estimates can force the theoretical bounds to rely on global Fourier projection operators \citep{Bourgain1993}. 
Consequently, the discrete nonlinear iteration is constrained by a strict Courant-Friedrichs-Lewy (CFL) condition that tightly couples the time step to the spatial grid resolution, limiting practical applicability for high-resolution simulations.

Recently, there has been a growing interest in integrating deep learning techniques with traditional iterative methods to accelerate convergence \citep{raissi2019pinn, lu2021deeponet, li2021fno, karniadakis2021physics}. These algorithms also aim to overcome the theoretical limitations of classical solvers \citep{wang2020understanding}.
In the context of large-scale linear systems and highly oscillatory partial differential equations, such as the indefinite Helmholtz equation, machine learning techniques have been successfully intertwined with multigrid (MG) and Krylov subspace methods \citep{hsieh2019learning, greenfeld2019learning, markidis2021the}.
These hybrid approaches employ deep neural networks to learn smoothers \citep{huang2022learning} and transfer operators \citep{luz2020learning}. Furthermore, they are used for coarse-grid corrections \citep{cui2022fourier, azulay2022multigrid, belbute2020combining}.
For instance, the Wave-ADR neural solver \citep{stanziola2021helmholtz} partitions the iterative error into characteristic and non-characteristic components. Classical multigrid wave cycles attenuate high-frequency errors, while neural networks address near-nullspace characteristic components on a coarse scale.
Similarly, the HINTS framework \citep{zhang2022hints, zhang2024hints, kahana2022geometry} leverages deep operator networks (DeepONet) and other continuous learning algorithms \citep{li2020neural, khoo2021solving}. These models then construct efficient preconditioners \citep{lu2021deeponet}, reducing low-frequency error components and leaving the high-frequency parts to standard stationary methods.
Other hybrid strategies, such as encoder-solver architectures \citep{he2019mgnet, um2020solver}, integrate convolutional neural networks with classical geometric solvers. These combinations map intractable preconditioner inverses \citep{sirignano2018dgm, han2018solving, e2017deep, raissi2018hidden}. In addition, several variants employ surrogate techniques to refine local approximations \citep{dong2021local, sun2020surrogate, margenberg2022neural}.
These hybrid iterative neural solvers demonstrate a useful spectral complementarity. Neural operators, equipped with implicit frequency-domain mappings, capture global continuous representations and reduce low-frequency error components that can stall classical relaxations.

Motivated by these limitations, we study a learning task for solver correction.
Given a differentiable one-step solver $S_\tau^{\rm base}$ and a reference flow $\Phi_\tau$, learn a time-step-scaled correction
\[
S_{\theta,\tau}(u)=S_\tau^{\rm base}(u)+C_\theta(u,\tau),
\qquad
C_\theta(u,\tau)=\tau H_{\mathrm{neural}}(u,\tau;\theta).
\]
The goal is to reduce long-horizon trajectory error while preserving the stability scale of the base solver.
HIN-LRI is this framework applied to low-regularity integrators for dispersive PDEs.

\paragraph{Main Results.}
First, we prove a conditional defect-propagation bound.
If $C_\theta$ approximates the one-step LRI defect on a compact data set, then the global error is reduced by the same measured defect ratio.
Second, we show that the learned component contributes only an $O(\tau L_{\theta,K})$ Lipschitz term, up to projection constants.
Third, we report diagnostics for the quantities used in the assumptions: defect ratio, latent dimension sensitivity, and spectral-norm bounds.

HIN-LRI embeds the learned correction into an alternating iteration, as shown in \cref{fig:hin_lri_flowchart}.
The LRI branch propagates the linear dispersion exactly and leaves a structured resonance defect.
The neural branch maps a latent residual to $C_\theta$.
The update keeps the LRI backbone and adds only this time-step-scaled correction.

\begin{figure}[htbp]
    \centering
    \includegraphics[width=\textwidth]{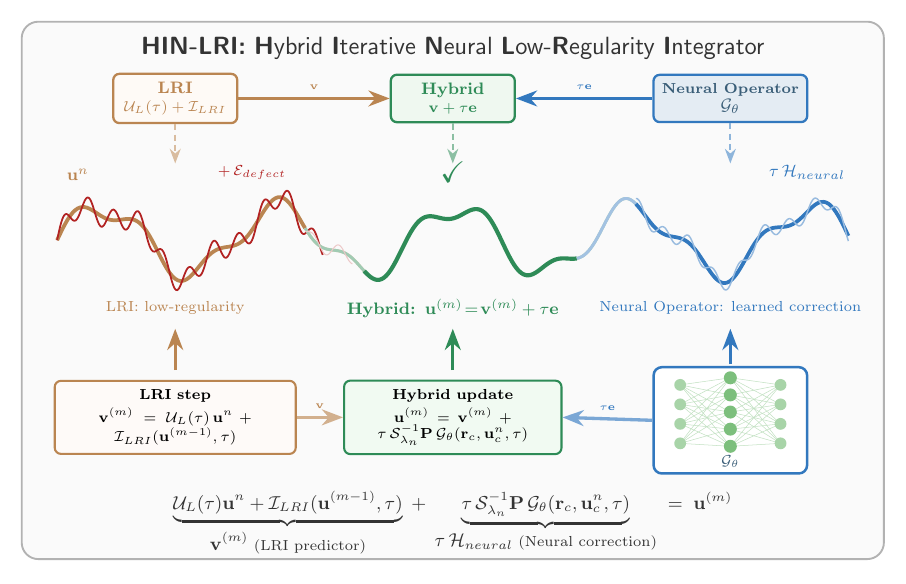}
    \vspace{-1.5em}
\caption{Schematic of the HIN-LRI alternating iteration, showing the LRI predictor, latent neural correction, and final update.}
\label{fig:hin_lri_flowchart}
\end{figure}

From the perspective of machine learning methodology, our main contributions are:
\begin{itemize}
    \item \textbf{A general residual-correction framework for learning-augmented numerical solvers.}
    We formalize a design pattern in which a classical iterative solver provides a consistent first-order approximation, and a neural operator acts exclusively on the structured residual that the solver cannot close analytically.
    The separation into a ``physical backbone + learned residual'' is not specific to dispersive PDEs: it applies whenever the solver's truncation error admits a well-characterised defect operator that can serve as a regression target.
    We provide an error decomposition (\cref{subsec:theoretical_analysis}) showing how the learned defect error $\varepsilon_{\rm learn}$ propagates from one step to the full trajectory.
    
    \item \textbf{Stability-preserving latent-space operator learning.}
    The neural correction operates on a low-dimensional orthogonal manifold ($K \ll N$) and carries an explicit time-step scaling $\tau$, so that its own Lipschitz contribution to the one-step map is $\tau L_{\theta,K}$ rather than $L_{\theta,K}$.
    Under the stated stability assumptions, this gives a bounded Gronwall factor. The CFL statement is a correction-level bound, not a full high-frequency stability proof.
    
    \item \textbf{Solver-in-the-loop (SITL) training aligned with multi-step dynamics.}
    Rather than training the neural component on isolated one-step regression targets, we unroll the full hybrid iteration and optimize a trajectory loss in a physically motivated function-space norm (discrete Bourgain space).
    This solver-aware training objective reduces distribution shift between training and deployment and provides a principled alternative to standard supervised operator learning. The approach generalises to any differentiable iterative solver.
    
    \item \textbf{Empirical validation and diagnostics.}
    We evaluate the framework on three low-regularity dispersive benchmarks. We compare against classical solvers and neural PDE surrogates. We also report ablations, out-of-distribution transfer, invariant-drift diagnostics, runtime, $\varepsilon_{\rm learn}$, and $L_{\theta,K}$ (\cref{app:reproducibility}).
\end{itemize}

The paper is organized as follows. \Cref{sec:preliminaries} introduces the mathematical setting and reviews the analytical background for low-regularity integrators. \Cref{sec:method} presents the HIN-LRI framework, including the alternating spectral-neural iteration, the solver-in-the-loop training, and the theoretical analysis. \Cref{sec:experiments} reports the numerical experiments. \Cref{sec:conclusions} concludes with a discussion of limitations and future work.

\section{Related Work}

\subsection{Low-Regularity Integrator}
Classical splitting and exponential integrators typically exhibit severe order reduction on rough data.
Low-regularity integrators (LRIs) mitigate this by embedding dominant nonlinear oscillations into the discretization via twisted variables and resonance expansions.
Beyond equation-specific constructions, general LRI frameworks avoiding reliance on Fourier series have been developed, supporting non-periodic domains and non-polynomial nonlinearities, alongside systematic high-order formalisms via decorated trees and forest formulae that enable symmetric designs \citep{RoussetSchratz2021Framework, BrunedSchratz2022Trees, BronsardBrunedMaierhoferSchratz2026Forest}.
For the cubic NLS, foundational exponential-type LRIs achieve first-order convergence by integrating the dominant nonlinear frequency interaction exactly \citep{OstermannSchratz2018}.
Second-order accuracy at reduced regularity has been addressed by resonance-aware schemes \citep{Knoller2019FourierNLS, Oster2022Order2NLS}, with new 1D variants pushing the $L^2$ theory below $H^2$ \citep{CaoLiLin2024SecondOrderNLS}.
At very low regularity, stability in $L^2$ has been obtained using discrete and continuous Strichartz/Bourgain-space estimates \citep{Oster2021ErrorNLS, Oster2023BourNLS, Ruff2025SpaceTimeBounds}.
Furthermore, structure-preserving symmetric and symplectic Runge--Kutta resonance-based LRIs have emerged to reconcile symplecticity with low-regularity convergence \citep{Bronsard2024SymmetricNLS, FengMaierhoferWang2025ExplicitSymmetricNLS, MaierhoferSchratz2025RKResonance}, while long-time behavior limits remain actively studied \citep{FengMaierhoferSchratz2024LongTime, Yao2022QuadraticNLS}.
For KdV, the foundational exponential-type integrator introduced first-order convergence by exploiting cubic dispersion resonance identities \citep{HofmanovaSchratz2017KdV}.
Advanced embedded exponential-type LRIs and methods analyzed via discrete Bourgain spaces subsequently optimized convergence under increasingly rough data, including solutions below $H^1$, using averaging approximations and perturbative low-regularity stability arguments \citep{WuZhao2022ELRIKdV, RoussetSchratz2022KdVLowReg, LiWu2025UnfilteredKdV}.

\subsection{Hybrid Iterative Neural Solvers}
Neural PDE solvers span fully learned surrogates like PINNs \citep{raissi2019pinn} or operator learners \citep{lu2021deeponet, li2021fno}, and hybrid iterative approaches that embed learning into numerically consistent outer loops to improve controlled accuracy.
A prominent coupling pattern exploits spectral bias by combining simple smoothers for high-frequency errors with neural corrections for slow modes, as seen in frameworks like HINTS \citep{zhang2024hints} and the Fourier Neural Solver \citep{cui2025fns}, though aligning training objectives with solver dynamics remains crucial for reliability \citep{wu2026reliability}. 
Beyond relaxation, learned preconditioners for Krylov methods reinterpret operator learning as a solver component, using DeepONet-based subspace corrections \citep{kopanicakova2025precond}, graph neural networks \citep{chen2025gnp}, or fixed low-rank coarse spaces \citep{benanti2026neuralsp}.
For challenging PDEs such as high-frequency Helmholtz equations, specialized multigrid/neural-network hybrids couple learned coarse corrections with classical smoothing \citep{azulay2023helmholtz, lerer2023compactimplicit, cui2025waveadrns}.
Related hybridization patterns also appear broadly in optimization and inverse problems, including deep unfolding \citep{gregor2010lista} and plug-and-play priors \citep{venkatakrishnan2013pnp, chan2016pnpadmm, romano2017red}.
Implicit-layer fixed-point models \citep{bai2019deq} further train equilibria via implicit differentiation, all driving toward the shared goals of certifiable convergence and predictable end-to-end computational costs.

HIN-LRI differs from these lines in its target and in the scope of its guarantees.
The learned component is not used as a black-box solver or generic smoother; it is trained to approximate the residual defect left by a low-regularity integrator, while the exact dispersive propagator remains part of every step.
\Cref{tab:method_positioning} summarizes this positioning.

\begin{table}[tbp]
    \centering
    \footnotesize
    \caption{Positioning of HIN-LRI relative to neural PDE solvers, hybrid iterative methods, and classical LRIs.}
    \label{tab:method_positioning}
    \begin{tabular}{@{}lccccc@{}}
    \toprule
    Method class & Exact propagator & Learns LRI defect & Stability claim & Low-reg. analysis & SITL \\
    \midrule
    FNO/DeepONet & No & No & No & No & Usually no \\
    HINTS & Yes & No & Partial & No & Partial \\
    Classical LRI & Yes & No & Analytical & Yes & No \\
    HIN-LRI & Yes & Yes & Conditional & Conditional & Yes \\
    \bottomrule
    \end{tabular}
\end{table}

\section{Preliminaries}
\label{sec:preliminaries}

In this section, we introduce the mathematical models of nonlinear dispersive equations and review the analytical background for low-regularity solutions.
We then outline the general framework of low-regularity integrators (LRIs) and discuss the inherent numerical challenges that motivate our hybrid neural-numerical approach.

\subsection{Problem Formulation}
\label{subsec:problem_formulation}

We consider a general class of nonlinear dispersive partial differential equations defined on a $d$-dimensional flat torus $\mathbb{T}^d = [0, 2\pi)^d$, which govern the spatiotemporal evolution of various wave phenomena.
The initial value problem is formulated as
\begin{equation}\label{eq:general_dispersive}
    \begin{cases}
        \partial_t u(t, \mathbf{x}) + i \mathcal{L}(\nabla) u(t, \mathbf{x}) = \mathcal{N}(u(t, \mathbf{x})), \quad \mathbf{x} \in \mathbb{T}^d, \ t \in (0, T], \\
        u(0, \mathbf{x}) = u_0(\mathbf{x}), \quad \mathbf{x} \in \mathbb{T}^d,
    \end{cases}
\end{equation}
where $u(t, \mathbf{x})$ denotes the real- or complex-valued wave field.
The operator $\mathcal{L}(\nabla)$ is a linear, self-adjoint pseudo-differential operator characterized by a real-valued dispersion relation $\omega(\mathbf{k})$, such that its action in the Fourier space is given by $\widehat{\mathcal{L} u}(\mathbf{k}) = \omega(\mathbf{k}) \hat{u}(\mathbf{k})$ for the discrete wavenumber $\mathbf{k} \in \mathbb{Z}^d$.
The term $\mathcal{N}(u)$ represents a polynomial nonlinearity. 
We use the convention $\mathcal{U}_L(t)=e^{-it\mathcal{L}}$.
Thus
\[
u(t+\tau)=\mathcal{U}_L(\tau)u(t)+\int_0^\tau \mathcal{U}_L(\tau-\sigma)\mathcal{N}(u(t+\sigma))\,d\sigma .
\]

This abstract formulation encapsulates several fundamental physical models widely studied in the literature.
For instance, setting the spatial dimension $d=1$, $i\mathcal{L}(\nabla) = \partial_x^3$ (corresponding to $\omega(k) = -k^3$), and $\mathcal{N}(u) = \frac{1}{2}\partial_x(u^2)$ yields the Korteweg--de Vries (KdV) equation.
Alternatively, choosing the Laplacian operator $\mathcal{L}(\nabla) = -\Delta$ (corresponding to $\omega(\mathbf{k}) = |\mathbf{k}|^2$) yields the quadratic and cubic nonlinear Schr\"odinger (NLS) equations when the nonlinearity is given by $\mathcal{N}(u) = \lambda u^2$ (or $\lambda |u|^2$) and $\mathcal{N}(u) = \lambda |u|^2 u$, respectively, with $\lambda \in \mathbb{R}$.

In many physical applications, the initial state $u_0$ is often highly oscillatory or strictly non-smooth, residing in a low-regularity Sobolev space $H^s(\mathbb{T}^d)$ with a critically small or even negative index $s \le 0$.
The analytical well-posedness of \cref{eq:general_dispersive} in such rough functional spaces relies heavily on advanced harmonic analysis tools, particularly the discrete Bourgain spaces $X_{s,b}$, equipped with the space-time norm
\begin{equation}
    \| w \|_{X_{s,b}}^2 := \int_{\mathbb{R}} \sum_{\mathbf{k} \in \mathbb{Z}^d} \langle \mathbf{k} \rangle^{2s} \langle \sigma - \omega(\mathbf{k}) \rangle^{2b} \big| \mathcal{F}_{t,\mathbf{x}} \{w\}(\sigma, \mathbf{k}) \big|^2 d\sigma,
    \label{eq:bourgain_norm}
\end{equation}
where $\langle \mathbf{k} \rangle = (1+|\mathbf{k}|^2)^{1/2}$ and $\mathcal{F}_{t,\mathbf{x}}$ denotes the spatiotemporal Fourier transform.
This functional space effectively isolates the linear dispersive wave propagation onto the characteristic manifold $\sigma = \omega(\mathbf{k})$, providing a crucial mechanism to recover the lost spatial regularity required by the nonlinear term $\mathcal{N}(u)$.
For example, the KdV equation is globally well-posed in $H^s$ for $s \ge -1$, while the cubic NLS equation is well-posed for $s \ge 0$.

Numerically solving \cref{eq:general_dispersive} under low-regularity conditions presents substantial challenges.
Applying the Fourier pseudo-spectral method on a uniform spatial grid yields a discrete projection operator $\Pi_N$, where $N$ represents the frequency truncation limit.
The continuous PDE is then reduced to a large-scale semi-discrete system
\begin{equation}
    \partial_t \mathbf{u}_N + i \mathcal{L}_N \mathbf{u}_N = \Pi_N \mathcal{N}(\mathbf{u}_N).
    \label{eq:linsys_dispersive}
\end{equation}
For rough data, the spatial derivatives embedded in $\mathcal{L}$ and $\mathcal{N}$ are essentially unbounded.
Classical time-marching methods, such as standard exponential integrators or operator splitting methods, inherently rely on the boundedness of high-order time derivatives of the exact solution, which translates to requiring high spatial regularity (e.g., $u_0 \in H^3$ or $H^4$).
When applied to solutions below these regularity thresholds, classical methods can fail to resolve the high-frequency oscillatory components, suffering from severe order reduction and spurious numerical instability.

\subsection{Notation}
\label{subsec:notation}

\Cref{tab:notation} summarizes the main symbols used throughout the paper.
We use $\mathcal{K}$ for compact sets and reserve $K$ for the latent dimension.
We use $N_x$ for spatial modes, $N_t$ for time steps, and $N_{\rm train}$ for training samples.
When legacy notation $N$ appears in a numerical scheme, it means $N_x$.

\begin{table}[tbp]
    \centering
    \footnotesize
    \caption{Main notation used in the HIN-LRI formulation and analysis.}
    \label{tab:notation}
    \begin{tabular}{@{}ll@{}}
    \toprule
    Symbol & Meaning \\
    \midrule
    $u(t)$, $u^n$, $u^{(m)}$ & Exact solution, time-step state, and inner Picard iterate \\
    $\tau$, $T$, $N_x$, $N_t$, $K$ & Time step, final time, grid size, time steps, and latent dimension \\
    $\mathcal{L}$, $\omega(k)$, $\mathcal{U}_L(t)$ & Dispersive operator, symbol, and exact linear propagator \\
    $\mathcal{N}$, $\mathcal{I}_{LRI}$ & Nonlinearity and base low-regularity integrator \\
    $\mathcal{E}_{defect}$ & Residual defect between the exact and LRI Duhamel terms \\
    $H_{\mathrm{neural}}$, $C_\theta$ & Network output and actual correction $C_\theta:=\tau H_{\mathrm{neural}}$ \\
    $\mathbf{R}$, $\mathbf{P}$, $\boldsymbol{\Phi}$ & Restriction, prolongation, and latent basis; $\mathbf{R}=\boldsymbol{\Phi}^*$ \\
    $\mathcal{S}_{\lambda}$ & Empirical normalization/scaling used before latent projection \\
    $\varepsilon_{\rm learn}$, $L_{\theta,K}$ & Learned defect ratio and latent Lipschitz constant \\
    \bottomrule
    \end{tabular}
\end{table}

\subsection{Low-Regularity Integrators}
\label{subsec:low_regularity_integrators}

To address the severe regularity requirements imposed by classical numerical methods, low-regularity integrators (LRIs) have been developed.
The fundamental concept behind LRIs is to isolate the stiff, highly oscillatory linear dispersive dynamics from the nonlinear interactions.
This is achieved by introducing the unitary continuous evolution group $\mathcal{U}_L(t) = \exp(-it\mathcal{L})$ and defining the twisted variable $v(t, \mathbf{x}) = \mathcal{U}_L(-t)u(t, \mathbf{x})$.
This Lawson-type transformation completely absorbs the linear differential operator, yielding an equivalent evolution equation driven purely by the frequency-modulated nonlinearity
\begin{equation}
    \partial_t v(t, \mathbf{x}) = \mathcal{U}_L(-t) \mathcal{N} \left( \mathcal{U}_L(t) v(t, \mathbf{x}) \right).
    \label{eq:twisted_evolution}
\end{equation}
Let $t_n = n\tau$ for $n=0, 1, \dots, T/\tau$ be a uniform temporal partition with step size $\tau$.
Integrating \cref{eq:twisted_evolution} over a temporal step interval $[t_n, t_{n+1}]$ provides the exact Duhamel integral formulation for the twisted variable.
In the Fourier space, the polynomial nonlinearity $\mathcal{N}$ induces a multi-dimensional convolution.
The nonlinear interaction of distinct frequency modes $\mathbf{k}_j$ generates highly oscillatory cross-resonance phase functions $\Phi$.
For instance, considering a generic nonlinearity of degree $p$, the exact evolution from $t_n$ to $t_{n+1}$ can be expressed at the Fourier mode level as
\begin{equation}
    \hat{v}_{\mathbf{k}}(t_{n+1}) = \hat{v}_{\mathbf{k}}(t_n) + \int_0^\tau \sum_{\mathbf{k}_1+\dots+\mathbf{k}_p=\mathbf{k}} \mathbf{C}(\mathbf{k}, \dots) e^{-is\Phi(\mathbf{k}, \mathbf{k}_1, \dots, \mathbf{k}_p)} \prod_{j=1}^p \hat{v}_{\mathbf{k}_j}(t_n+s) \, ds,
    \label{eq:duhamel_fourier}
\end{equation}
where $\mathbf{C}(\mathbf{k}, \dots)$ denotes the coefficient multiplier stemming from spatial derivatives in $\mathcal{N}$, and $\Phi(\mathbf{k}, \dots) = \omega(\mathbf{k}) - \sum_{j=1}^p \omega(\mathbf{k}_j)$ is the resonance phase function governing the nonlinear frequency coupling.
For the KdV equation, $\Phi = k^3 - k_1^3 - k_2^3$, while for the cubic NLS equation, $\Phi = |\mathbf{k}|^2 + |\mathbf{k}_1|^2 - |\mathbf{k}_2|^2 - |\mathbf{k}_3|^2$.

To construct a practical and explicitly computable scheme without releasing spatial derivatives, LRIs freeze the slowly varying term $v(t_n+s) \approx v(t_n)$ over the short interval $s \in [0, \tau]$, and analytically evaluate the dominant high-frequency oscillatory integral $\int_0^\tau e^{-is\Phi_{dom}} ds$.
By avoiding straightforward Taylor expansions on the exponential phase, LRIs successfully circumvent the regularity paradox.
For the cubic NLS equation, extracting the dominant phase component $\Phi_{dom} = -2|\mathbf{k}_1|^2$ leads to the classical first-order resonance-based scheme
\begin{equation}
    u^{n+1} = \mathcal{U}_L(-\tau) \left[ u^n - i\tau\lambda (u^n)^2 \left( \varphi_1(-2i\tau\Delta)\overline{u^n} \right) \right],
\end{equation}
where $\varphi_1(z) = (e^z - 1)/z$ acts as an exact filter for the resonance frequency.
For the KdV equation, the algebraic identity $k^3 - k_1^3 - k_2^3 = 3kk_1k_2$ enables the exact integration of the resonance phase to perfectly cancel the singular derivative multiplier $ik$, yielding the baseline first-order LRI propagator
\begin{equation}
    \mathcal{I}_{LRI}(u^n, \tau) = \frac{1}{6} \mathbb{P} \left[ \left(\mathcal{U}_L(-\tau) \partial_x^{-1} u^n \right)^2 \right] - \frac{1}{6} \mathcal{U}_L(-\tau) \mathbb{P} \left[ \left(\partial_x^{-1} u^n \right)^2 \right],
    \label{eq:kdv_lri_base}
\end{equation}
where $\mathbb{P}$ denotes the projection onto mean-zero functions, and $\partial_x^{-1}$ is the pseudo-differential anti-derivative operator.

For higher-order methods and nested multi-wave interactions, exact phase factorization is often unavailable.
Analytical LRIs then use Taylor expansions, filtered midpoint rules, or temporal averaging.
These approximations leave a residual defect $\mathcal{E}_{defect}(\mathbf{u}^n, \tau)$.
The complete discrete evolution governing the numerical solution can thus be formulated as a large-scale nonlinear system
\begin{equation}
    \mathbf{u}^{n+1} = \Pi_N \mathcal{U}_L(-\tau) \mathbf{u}^n + \mathcal{I}_{LRI}(\mathbf{u}^n, \tau) + \mathcal{E}_{defect}(\mathbf{u}^n, \tau).
    \label{eq:discrete_evolution_defect}
\end{equation}

Relying only on harmonic-analysis bounds introduces several constraints.
Taylor truncations reintroduce derivative loss, demanding higher regularity than the PDE requires.
Averaging approximations, such as $M_\tau(e^{is(\phi_1+\phi_2)}) \approx M_\tau(e^{is\phi_1}) M_\tau(e^{is\phi_2})$, leave a mismatch kernel that enforces logarithmic penalties $\mathcal{O}(\tau^\gamma \ln(1/\tau))$ on the truncation error, capping the achievable accuracy.
Moreover, discrete Bourgain space analyses necessitate strict CFL conditions, such as $\tau \le \mathcal{O}(N^{-3})$, to maintain local Lipschitz stability in extreme low-regularity spaces.
These constraints limit standard LRIs in rough, high-order, and high-resolution regimes.
They motivate a learned residual corrector for $\mathcal{E}_{defect}$.

\section{Method}
\label{sec:method}

In this section, we describe the proposed methodology.
\Cref{subsec:motivation} mathematically analyses the analytical defects of classical low-regularity integrators.
\Cref{subsec:hinlri} presents the Hybrid Iterative Neural Low-Regularity Integrator (HIN-LRI).
It combines an explicit dispersive step with a latent neural residual correction.
\Cref{subsec:algorithms} introduces the algorithmic implementation and the solver-in-the-loop (SITL) training in discrete Bourgain spaces.
\Cref{subsec:theoretical_analysis} gives the conditional error and stability analysis.

\subsection{Motivation}
\label{subsec:motivation}

We consider a general class of nonlinear dispersive equations defined on the torus $\mathbb{T}^d$ or whole space $\mathbb{R}^d$:
\begin{equation} \label{eq:general_pde}
    \partial_t u + i\mathcal{L}(\nabla)u = \mathcal{N}(u), \quad u(0,x) = u_0 \in H^s,
\end{equation}
where $\mathcal{L}(\nabla)$ is a self-adjoint linear dispersive operator with symbol $\omega(\mathbf{k}) \sim \mathcal{O}(|\mathbf{k}|^\alpha)$ (e.g., $\alpha=2$ for NLS, $\alpha=3$ for KdV), and $s \le 0$ indicates very low regularity.
The foundational paradigm of classical LRIs relies on the Lawson transform (twisted variable) $v(t) = e^{it\mathcal{L}}u(t)$, which analytically absorbs the high-frequency linear stiffness into a purely oscillatory Duhamel integral:
\begin{equation} \label{eq:exact_duhamel}
    \hat{v}_{\mathbf{k}}(t_n+\tau) = \hat{v}_{\mathbf{k}}(t_n) + \int_0^\tau \sum_{\mathbf{k} = \sum \pm \mathbf{k}_j} e^{is\Phi(\mathbf{k}, \mathbf{k}_j)} \widehat{\mathcal{N}}(\mathbf{v}(t_n+s)) ds.
\end{equation}
To achieve high-order convergence without requiring bounded high-order spatial derivatives of $u$, classical LRIs aim to replace the exact oscillatory integral in \cref{eq:exact_duhamel} by an analytically computable resonance quadrature:
\begin{equation}
    \hat{v}_{\mathbf{k}}^{n+1}
    =
    \hat{v}_{\mathbf{k}}^{n}
    +
    \mathcal{Q}_{\tau,\mathbf{k}}^{\Phi}(\hat{v}^{n}),
    \qquad
    \mathcal{Q}_{\tau,\mathbf{k}}^{\Phi}(\hat{v}^{n})
    \approx
    \int_0^\tau
    \sum_{\mathbf{k}=\sum \pm \mathbf{k}_j}
    e^{is\Phi(\mathbf{k},\mathbf{k}_j)}
    \widehat{\mathcal{N}}(\hat{v}^{n})\,ds .
\end{equation}
The entire LRI design problem is therefore reduced to constructing a stable and fast approximation $\mathcal{Q}_{\tau}^{\Phi}$ to the resonant Duhamel operator.
In low regularity, however, the Fourier coefficients of $v$ do not decay fast enough to absorb uncontrolled multipliers. Algebraic phase factorizations may fail. Taylor or filter expansions may release powers of $|\mathbf{k}|^\alpha$. Phase-decoupling residuals may accumulate logarithmic losses. Implicit structure-preserving variants may require fixed-point maps of the form $\hat v^{n+1}=\hat v^n+\mathcal{Q}_{\tau}^{\Phi}(\hat v^n,\hat v^{n+1})$ whose contraction constant grows with the rough Sobolev norm.
Higher order, multidimensionality, and exact symplecticity make these analytical requirements harder to satisfy.
The most representative numerical signatures of these four failures are summarized in \cref{fig:lri_four_defects_diagnosis}. Panel~(a) shows the brittleness of algebraic resonance factorization. Panel~(b) shows derivative-loss amplification induced by iterated multipliers. Panel~(c) shows the aliasing/logarithmic cascade behind the CFL-type restriction. Panel~(d) shows the divergence or costly saturation of implicit Picard iterations.
The detailed diagnostic figures are deferred to \cref{subsec:lri_defect_verification}.

\begin{figure}[htbp]
    \centering
    \includegraphics[width=\textwidth]{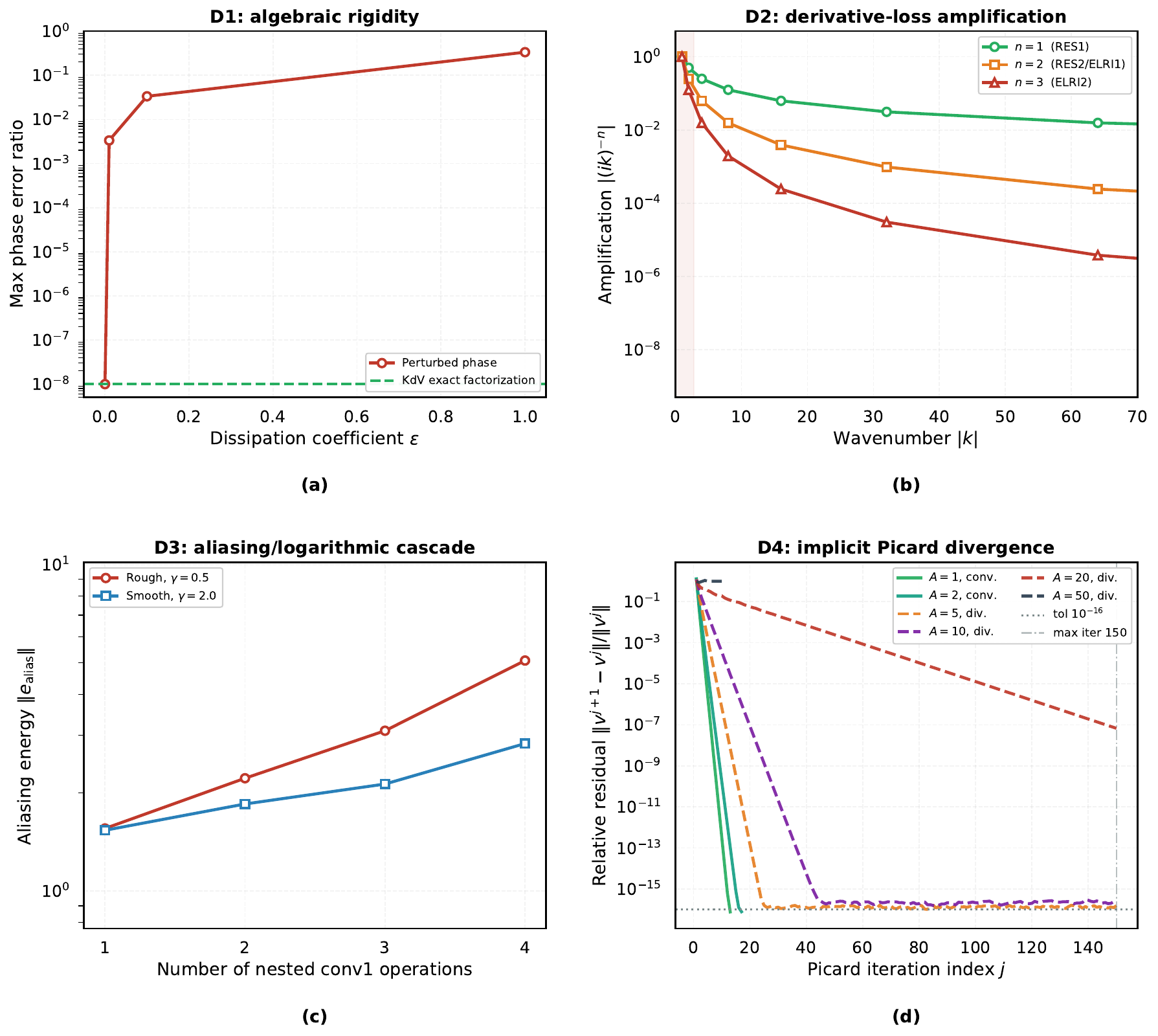}
    \vspace{-1.5em}
    \caption{Numerical diagnostics illustrating algebraic rigidity, derivative loss, CFL effects, and structure drift in analytical LRIs.}
    \label{fig:lri_four_defects_diagnosis}
\end{figure}

\subsubsection{Regularity Barrier from Explicit Taylor Truncation}
To construct high-order schemes or handle non-factorizable cross-resonances, one often approximates the convolution of $e^{is\Phi}$.
Classical methods split the phase operator into a dominant integrable part and a lower-order part: $\mathcal{L} = \mathscr{L}_{dom} + \mathscr{L}_{low}$.
The residual oscillation $e^{is\mathscr{L}_{low}}$ is then evaluated via an explicit Taylor series expansion:
\begin{equation} \label{eq:taylor_expansion}
    e^{is\mathscr{L}_{low}} = \sum_{m=0}^{r-1} \frac{(is\mathscr{L}_{low})^m}{m!} + \mathcal{R}_r(s\mathscr{L}_{low}), \quad \text{where } \left\| \mathcal{R}_r(s\mathscr{L}_{low}) \right\|_{op} \le \frac{s^r \|\mathscr{L}_{low}\|_{op}^r}{r!}.
\end{equation}
This introduces a regularity trade-off.
While the unitary group $e^{is\mathcal{L}}$ is a bounded isometry on $H^s$, its polynomial expansion explicitly releases the unbounded spatial differential operator $\mathscr{L}_{low}^r$ into the local truncation error (LTE).
In the decorated tree formalism, the approximation operator $\Pi^{n,r}$ yields an exact local error bound for a tree $T$:
\begin{equation} \label{eq:derivative_loss}
    \| (\Pi - \Pi^{n,r})T(\tau) \|_{H^\gamma}^2 = \sum_{\mathbf{k}} \langle \mathbf{k} \rangle^{2\gamma} \left| \mathcal{F}_{\mathbf{k}} \left[ \mathcal{R}_r \hat{v} \right] \right|^2 \propto \tau^{2r+4} \sum_{\mathbf{k}} \langle \mathbf{k} \rangle^{2\gamma} |\mathbf{k}|^{2\alpha r} |\hat{v}_{\mathbf{k}}|^2 = \tau^{2r+4} \|\nabla^{\alpha r} v\|_{H^\gamma}^2.
\end{equation}
Consequently, the local error $\mathcal{E}_{Taylor} = \mathcal{O}\left( \tau^{r+2} \|u\|_{H^{\gamma+\alpha r}} \right)$ requires $\gamma+\alpha r$ spatial derivatives.
For sub-$H^1$ rough data, $\|\nabla^{\alpha r} v\|_{H^\gamma}$ can be very large.
This produces the order reduction shown in \cref{fig:taylor_derivative_loss}.
On smooth initial data ($\gamma = 3.0$), classical integrators achieve their nominal orders, with KdV-ETD1 attaining order $0.99$ and NLS-Strang splitting attaining $2.18$.
On rough data ($\gamma = 0.5$), KdV-ETD1 drops to order $0.26$ and KdV-Lawson1 diverges.
NLS-Strang drops from $2.18$ to $0.85$.
NLS-BS22 drops from $1.87$ on smooth data to $0.58$ on rough data.
KdV-RES1 drops to slope $0.18$.
These results show that current LRI schemes can lose their nominal rates on rough data.
NLS-RES1 is more robust in this test, retaining order $1.02$.
The spectral error in \cref{fig:taylor_derivative_loss_b} shows the same pattern: high-frequency modes carry most of the one-step error.

\subsubsection{Algebraic Rigidity and Combinatorial Growth}
To map frequency-domain convolutions back to physical space and retain $\mathcal{O}(N \log N)$ FFT efficiency, classical LRIs depend on exceptional algebraic factorizations (e.g., $k^3 - k_1^3 - k_2^3 \equiv 3k k_1 k_2$ for 1D KdV).
Once this identity is perturbed, the fast FFT-based reduction loses its justification and the numerical implementation reverts toward direct multi-index summation.
At higher order, the scheme must additionally evaluate multi-linear contributions over the Hopf algebra of decorated trees $\mathfrak{T}_0^{r+2}(R)$. Each tree node encodes a nonlinear interaction and each edge carries a frequency label. The full numerical scheme sums over all such trees up to a given depth:
\begin{equation} \label{eq:decorated_tree}
    U_{\mathbf{k}}^{n,r}(\tau,v) = \sum_{T \in \mathfrak{T}_0^{r+2}(R)} \frac{\Upsilon^p(\lambda_{\mathbf{k}}T)(v)}{S(T)} \Pi^n\left(\mathscr{D}^r(\mathcal{I}_{(t_1,0)}(\lambda_{\mathbf{k}}T))\right)(\tau).
\end{equation}
The cardinality $|\mathfrak{T}_0^{r+2}(R)|$ grows rapidly with the approximation order $r$, so even when the underlying FFT structure survives, higher-order LRIs accumulate a substantial combinatorial overhead.
The broader multidimensional complexity growth of generalized convolutions remains a genuine theoretical concern, but \cref{fig:algebraic_rigidity} directly quantifies the two effects that are numerically visible here: algebraic brittleness and higher-order combinatorial inflation.
Panel~(a) shows that the KdV identity $k^3{-}k_1^3{-}k_2^3 \equiv 3kk_1k_2$ is numerically stable for pure KdV ($\text{residual} < 10^{-8}$). A KdV-Burgers perturbation with $\varepsilon = 0.01$ creates a residual of $\sim 10^{-3}$ at $|k|=2$.
Panel~(b) shows the direct computational consequence. The FFT-based convolution achieves an empirical growth of $N^{0.46}$ (consistent with $\mathcal{O}(N\log N)$), while the brute-force fallback scales as $\mathcal{O}(N^{1.92}) \approx \mathcal{O}(N^2)$, producing an $85{\times}$ speed gap at $N=512$ that widens to $10^{3}{\times}$ at $N=4096$.
Panels~(c) and~(d) expose the higher-order combinatorial penalty. RES1 requires only $2$ \texttt{conv1} calls per step, while ELRI1 and ELRI2 each require $12$ ($6{\times}$) and ULRI requires $9$ ($4.5{\times}$). The decorated-tree term count grows as $\propto p^{2.5}$ and reaches $132$ terms at order $p=5$.
Together, these measurements show that even in the one-dimensional setting where the FFT factorization exists, higher-order LRIs already suffer from substantial implementation and runtime inflation. The figure should be interpreted as numerical evidence for algebraic rigidity and combinatorial explosion, rather than as a direct benchmark of the full multidimensional curse of dimensionality.

\subsubsection{Residual Resonance Mismatch and CFL-Type Restrictions}
Intuitively, a nonlinear dispersive PDE creates interactions between waves at different frequencies.
Tracking all interactions would require evaluating oscillatory integrals whose phases depend on combined frequencies.
Unfiltered LRIs approximate this by factoring the integral into independent single-frequency averages---a simplification that works well when the interacting frequencies are well-separated but leaves a residual error (the \emph{phase mismatch kernel} $\boldsymbol{\eta}$) whenever they are not.
This residual is not removed by the current algebraic decoupling and contributes to the logarithmic error term and CFL-type restriction analysed below.

To circumvent the order reduction caused by high-frequency hard filters, unfiltered LRIs (ULRIs) decouple nested exponential phases via time-averaging approximations:
\begin{equation} \label{eq:average_approx}
    \int_0^\tau e^{-is(\phi_1+\phi_2)} ds = \tau \mathcal{M}_\tau\left(e^{-is\phi_1}\right)\mathcal{M}_\tau\left(e^{-is\phi_2}\right) + \tau \boldsymbol{\eta}(\tau, \mathbf{k}),
\end{equation}
where $\mathcal{M}_\tau(f) = \frac{1}{\tau}\int_0^\tau f(s)ds$.
This decoupling leaves a non-zero phase mismatch kernel $\boldsymbol{\eta}(\tau, \mathbf{k})$.
The residual defect $\mathcal{E}_{defect} = \mathcal{I}_{exact} - \mathcal{I}_{LRI}$ takes the form:
\begin{equation}
    \mathcal{F}_{\mathbf{k}}[\mathcal{E}_{defect}] = - \sum_{\mathbf{k}_1+\mathbf{k}_2+\mathbf{k}_3=\mathbf{k}} \frac{\tau}{18i\mathbf{k}} e^{-it_n \phi} \boldsymbol{\eta}(\tau, \mathbf{k}) \hat{v}_{\mathbf{k}_1}\hat{v}_{\mathbf{k}_2}\hat{v}_{\mathbf{k}_3}.
\end{equation}
Through the logarithmically growing trilinear estimate on $L^2$, current analyses give the bound:
\begin{equation} \label{eq:log_penalty}
    \| \mathcal{E}_{defect}(u) \|_{L^2}
    \lesssim
    \mathcal{C} \tau^{1+\gamma}
    \left( \sum_{0 < |\mathbf{k}| \le \tau^{-1}} \frac{1}{|\mathbf{k}|} \right)
    \|u\|_{H^\gamma}^3
    \lesssim
    \tau^{1+\gamma}\ln\frac{1}{\tau}.
\end{equation}
This logarithmic factor is visible in the measured truncation error.
Furthermore, in endpoint Bourgain spaces $X_{s, b}$, the lack of additional smallness forces a global spectral truncation $\Pi_N$.
The discrete Banach contraction dictates the Lipschitz bound $\text{Lip}(\mathcal{I}_{LRI}) \le \mathcal{C}\tau \|\partial_x \Pi_N\|_{op} \le \mathcal{C}\tau N^\alpha < 1$.
This yields the CFL-type restriction $\tau \le \mathcal{O}(N^{-\alpha})$ in that analysis.
These ULRI defects are directly measured in \cref{fig:ulri_defect}, while the filtering-based alternative is quantified separately in \cref{fig:filtering_cap}.
Panel~(a) of \cref{fig:ulri_defect} shows that on rough $H^{0.5}$ KdV data, the ULRI convergence curve is enveloped by $\tau^\gamma \ln(1/\tau)$ rather than $\tau^\gamma$. The logarithmic overhead persists uniformly across all tested step sizes and prevents ULRI from matching the clean $\mathcal{O}(\tau^\gamma)$ rate of RES1 and ELRI1.
Panel~(b) gives a numerical diagnostic of the CFL-type restriction. At fixed $\tau = 10^{-3}$, RES1 produces decreasing $L^2$ error as $N$ increases, while ULRI's error stagnates and then diverges near $N^* = (2\pi/\tau)^{1/3} \approx 39$.
Panel~(d) reveals the cost-accuracy trade-off. ULRI requires $27$ FFT-equivalent transforms per step ($4.5{\times}$ RES1's $6$), yet occupies the worst Pareto quadrant with simultaneously the highest error and the second-highest cost among the three methods tested.
Complementarily, \cref{fig:filtering_cap} shows what happens when one attempts to suppress the logarithmic growth through hard high-frequency filtering. The $\varphi_1$ filter in NLS-RES1 introduces a sinc-like spectral attenuation $|\varphi_1(-2hk^2)| = \sin(hk^2)/(hk^2)$, visible in panel~(b) for $h \in \{0.0625, 0.125, 0.25\}$, and induces a clear regularity-dependent convergence ceiling. RES1 exhibits empirical orders $0.192$, $0.608$, $0.921$, and $1.035$ for $\gamma = 0.25, 0.5, 1.0, 2.0$. BS22 yields $0.391$, $0.523$, $1.011$, and $1.917$. The filter penalty dominates the rough-data regime and erodes the nominal high-order gain.
At $\gamma = 0.5$, BS22 achieves only order $0.523$ despite being nominally second-order, while RES1 drops to $0.608$. Both methods are driven close to a half-order regime, so the logarithmic mismatch is not removed but traded for a filter-induced spectral ceiling.

\subsubsection{Implicit Cost of Symplectic Structure Preservation}
Classical asymmetric LRIs break time-reversal symmetry, leading to secular energy drift.
To rigorously preserve the symplectic two-form $\omega = \sum d\xi_{\mathbf{k}} \wedge d\eta_{\mathbf{k}}$ over long-time evolution, Runge-Kutta resonance schemes introduce $S$ internal stages:
\begin{equation} \label{eq:RK_stages}
    K_{p,q,r} = \mathcal{F}_p\left(\tau; c_q; u^n + \tau \sum_{\tilde{p},\tilde{q},\tilde{r}=0}^S a_{p,q,r}^{\tilde{p},\tilde{q},\tilde{r}} K_{\tilde{p},\tilde{q},\tilde{r}}\right).
\end{equation}
For this discrete mapping to preserve the quadratic invariants exactly, the real-valued coefficients $b^{p,q,r}$ and $a_{p,q,r}^{\tilde{p},\tilde{q},\tilde{r}}$ must satisfy the strict algebraic geometric condition $b^{\tilde{p},\tilde{q},\tilde{r}} b^{p,q,r} = b^{p,q,r} a_{p,q,r}^{\tilde{p},\tilde{q},\tilde{r}} + b^{\tilde{p},\tilde{q},\tilde{r}} a_{\tilde{p},\tilde{q},\tilde{r}}^{p,q,r}$.
By evaluating the diagonal entries ($p=\tilde{p}, q=\tilde{q}, r=\tilde{r}$), it algebraically demands $(b^{p,q,r})^2 = 2 b^{p,q,r} a_{p,q,r}^{p,q,r}$.
For any consistent method where $b \neq 0$, this yields:
\begin{equation} \label{eq:implicit_proof}
    a_{p,q,r}^{p,q,r} = \frac{1}{2} b^{p,q,r} \neq 0.
\end{equation}
Under these RK resonance conditions, exact preservation requires implicit coupling.
Solving the resulting nonlinear systems by fixed-point iteration adds $\mathcal{O}(M_{iter} N \log N)$ work per step.
The numerical consequences of non-conservation and the implicit trap are jointly quantified in \cref{fig:energy_drift}.
Panel~(a) tracks the normalized mass $M(t)/M(0)$ over $T=20$ at $\tau=0.05$. The Lie and Strang splittings hold $M/M(0) \equiv 1$ to machine precision, while RES1 accumulates a visible secular drift, confirming that explicit LRI updates violate the symplectic mass invariant at rate $\mathcal{O}(\tau)$ per unit time.
Panel~(b) reveals the complementary Hamiltonian picture. RES1 drifts at $\mathcal{O}(\tau)$ and Strang splitting at $\mathcal{O}(\tau^2)$. No explicit method simultaneously preserves both invariants exactly.
Panels~(c) and~(d) quantify the drift rates via log-log regression over $\tau \in \{0.2, 0.1, 0.05, 0.025, 0.0125\}$ at $T=5$. RES1's mass and Hamiltonian drift both fit slope ${\approx}1.0$, consistent with $\mathcal{O}(\tau)$. The Strang Hamiltonian drift fits slope ${\approx}2.0$ ($\mathcal{O}(\tau^2)$), consistent with its second-order conservation of $H$.
The algebraic proof in \cref{eq:implicit_proof} explains the root cause. Exact simultaneous preservation of both $M$ and $H$ enforces $a_{p,q,r}^{p,q,r} = \tfrac{1}{2}b^{p,q,r} \ne 0$, mandating a fully implicit solver. The implicit bottleneck is not an engineering shortcoming but a mathematical inevitability. Long-time simulations with explicit LRIs must either accept unbounded drift or pay the full fixed-point iteration cost.

\subsection{Hybrid Iterative Neural Low-Regularity Integrator (HIN-LRI)}
\label{subsec:hinlri}

\begin{figure}[htbp]
\centering
\includegraphics[width=\textwidth]{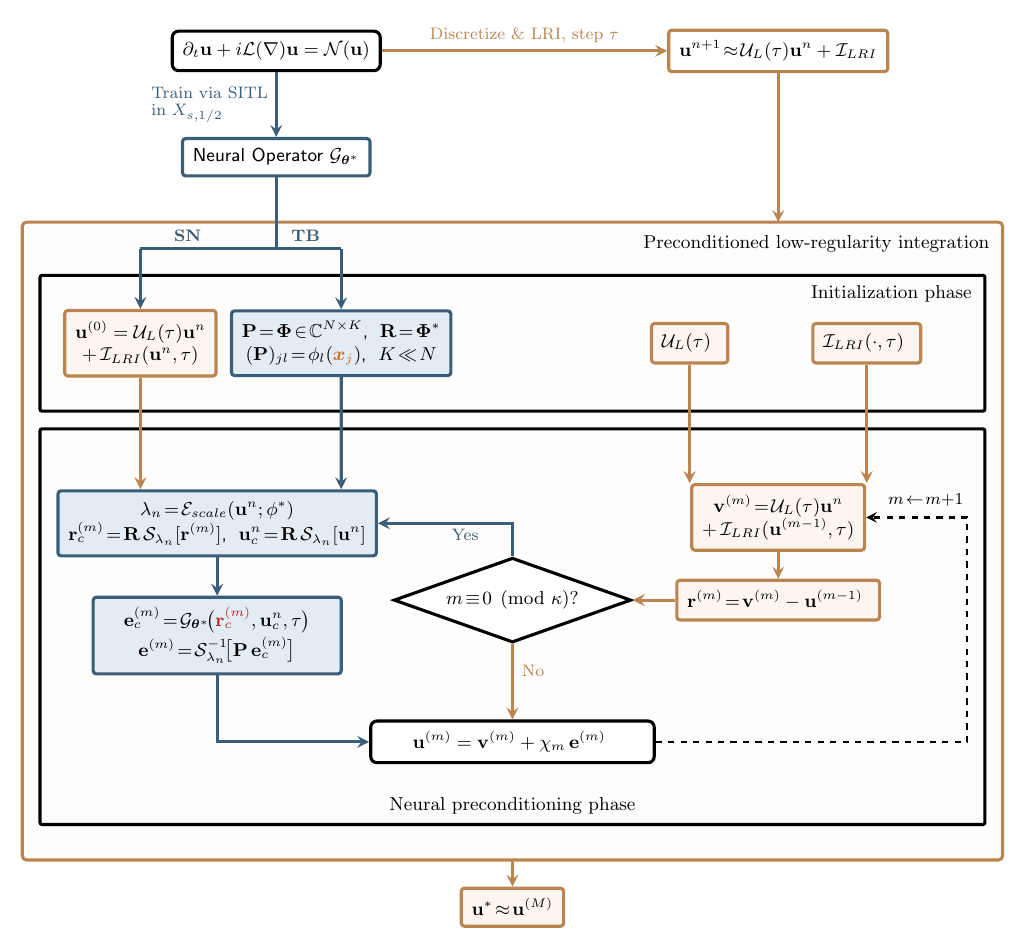}
\vspace{-1.5em}
\caption{HIN-LRI framework with an LRI spectral step followed by a time-step-scaled latent neural correction.}
\label{fig:hin_lri_framework}
\end{figure}

We now present the complete mathematical model of the HIN-LRI framework, illustrated in Figure~\ref{fig:hin_lri_framework}.
The model is built from precise operator definitions, tensor mappings, and functional analysis in Bourgain spaces.

\subsubsection{Intuitive Overview}
Before the formal development, we outline the key idea informally.
A standard LRI advances the solution by one time step using an analytical formula that exactly handles the linear dispersion but only approximately resolves the nonlinear interactions.
The approximation error---the \emph{resonance defect}---depends on how well one can factor or truncate certain oscillatory integrals in Fourier space.
Classical methods attempt this analytically, but the resulting algebraic identities are fragile, equation-specific, and introduce spurious derivative requirements.
HIN-LRI takes a different route: it keeps the analytical LRI step as a ``first draft'' of the solution and then adds a small learned correction.
This correction is computed by a neural operator that (i) compresses the solution onto a low-dimensional manifold of $K \ll N$ modes, (ii) maps the compressed residual through a lightweight network, and (iii) projects the result back to the full grid.
Because the learned correction is evaluated in a low-dimensional latent space, its Lipschitz bound depends on $K$ and on the projection constant, not directly on the full $N_x$-point grid.
This does not by itself prove that the entire scheme is independent of all high-frequency effects; the stability statement below is conditional on the base LRI stability and on the empirical quality of the learned defect approximation.
The network weights are trained end-to-end using a \emph{solver-in-the-loop} objective that unrolls multiple time steps and penalizes the trajectory error in a Bourgain-type norm. The learned correction is therefore optimized for the dynamics of the full time-stepping scheme rather than for a single-step regression target.

\subsubsection{Transition to the HIN-LRI Paradigm}
The analytical bounds reviewed above suggest that purely analytical LRI approaches face significant barriers when pushed to higher order, very rough data, and structure preservation simultaneously.
To address these challenges, we introduce the Hybrid Iterative Neural Low-Regularity Integrator (HIN-LRI).
By retaining the exact, zero-dissipation linear propagator $\mathcal{U}_L(\tau) = e^{-i\tau\mathcal{L}}$ as a structure-preserving physical backbone, we strategically embed a neural operator $\mathcal{G}_\theta$ mapped onto a latent manifold to execute a targeted residual correction:
\begin{equation} \label{eq:neural_correction_motivation}
    C_\theta(u^n,\tau) := \tau\,H_{\mathrm{neural}}(u^n,\tau;\theta)
    = \tau\,\mathbf{P} \circ \mathcal{G}_\theta \circ \mathbf{R}[u^n]
    \approx \mathcal{E}_{defect}(u^n,\tau).
\end{equation}
This spectral-neural alternation transforms the paradigm from analytical truncation to manifold projection.
The neural operator learns the integral mapping implicitly, avoiding the Taylor expansion that drives the regularity paradox.
Global frequency mixing sidesteps the combinatorial growth of decorated tree expansions.
End-to-end SITL optimization encourages the neural weights to act as an adaptive sub-grid filter, targeting the mismatch kernel $\boldsymbol{\eta}(\tau, \mathbf{k})$ on the training distribution.
Finally, the neural residual compensation provides a structural correction via a single explicit forward pass, avoiding the cost of fully implicit internal stages.

\subsubsection{Exact Evolution and Analytical Resonance Defect}
The formulation begins with the initial value problem for nonlinear dispersive equations:
\begin{equation}
    \partial_t u + i \mathcal{L} u = \mathcal{N}(u), \quad \mathbf{x} \in \mathbb{T}^d, \ t \in [0, T], \quad u(\mathbf{x},0) = u_0 \in H^s(\mathbb{T}^d) \quad (s \le 0),
\end{equation}
where the symbol of $\mathcal{L}$ is given by $\widehat{\mathcal{L} u}(\mathbf{k}) = \omega(\mathbf{k}) \hat{u}(\mathbf{k})$ with $\omega(\mathbf{k}) \in \mathbb{R}$.
The exact Duhamel integral formulation defines the linear propagator $\mathcal{U}_L(t) := \exp(-i t \mathcal{L}) = \mathcal{F}_{\mathbf{k}}^{-1} \left\{ e^{-i t \omega(\mathbf{k})} \mathcal{F}_{\mathbf{x}} \{\cdot\} \right\}$, which yields the exact evolution
\begin{equation}
    u(t_{n+1}) = \mathcal{U}_L(\tau) u(t_n) + \underbrace{\int_0^\tau \mathcal{U}_L(\tau-\sigma) \mathcal{N}\Big( \mathcal{U}_L(\sigma) \big[ \mathcal{U}_L(-t_n)u(t_n+\sigma) \big] \Big) d\sigma}_{\mathcal{I}_{exact}(u(t_n), \tau)}.
\end{equation}
Assuming $\mathbf{u}^n \approx u(t_n)$, the base LRI operator and the residual defect $\mathcal{E}_{defect}$ are defined as
\begin{equation}
    \mathcal{H}_{phys}(\mathbf{w}; \mathbf{u}^n) := \mathcal{U}_L(\tau) \mathbf{u}^n + \mathcal{I}_{LRI}(\mathbf{w}, \tau),
\end{equation}
\begin{equation}
    \mathcal{E}_{defect}(\mathbf{u}^n, \tau) := \mathcal{I}_{exact}(\mathbf{u}^n, \tau) - \mathcal{I}_{LRI}(\mathbf{u}^n, \tau).
\end{equation}
The Fourier spectral anatomy of the defect reveals the cross-resonance phase $\Phi(\mathbf{k}, \mathbf{k}_1, \dots, \mathbf{k}_p) := \omega(\mathbf{k}) - \sum_{j=1}^p \omega(\mathbf{k}_j)$.
In the Fourier domain, the defect expands as
\begin{equation}
    \mathcal{F}_{\mathbf{k}} \left[ \mathcal{E}_{defect}(\mathbf{u}^n, \tau) \right] = \sum_{\sum \mathbf{k}_j = \mathbf{k}} \mathbf{K}(\mathbf{k}) e^{-i t_n \Phi(\mathbf{k})} \underbrace{ \boldsymbol{\eta}(\tau, \mathbf{k}, \mathbf{k}_1, \dots, \mathbf{k}_p) }_{\text{Phase Mismatch Kernel}} \prod_{j=1}^p \hat{u}_{\mathbf{k}_j}^n.
\end{equation}
The phase mismatch kernel evaluates to 
\begin{equation}
    \boldsymbol{\eta} := \frac{1}{\tau}\int_0^\tau e^{-i\sigma \Phi(\mathbf{k})} d\sigma - \prod_{j=1}^{p-1} \left( \frac{1}{\tau}\int_0^\tau e^{-i\sigma \Phi_j(\mathbf{k})} d\sigma \right) \neq 0.
\end{equation}
This persistent non-zero kernel enforces the fundamental mathematical barriers of pure analytical LRIs, specifically the logarithmic accumulation penalty 
\begin{equation}
    \inf_{\mathbf{u} \in H^s} \left\| \mathcal{E}_{defect}(\mathbf{u}, \tau) \right\|_{L^2} \ge \mathcal{C} \tau^{1+\gamma} \ln\left(\frac{1}{\tau}\right)
\end{equation}
and the global $\Pi_N$ truncation CFL-type condition
\begin{equation}
    \sup_{\mathbf{u} \in X_{s, 1/2}} \text{Lip}(\mathcal{I}_{LRI}) > 1 \implies \tau \le \mathcal{O}(N^{-\alpha}).
\end{equation}

\subsubsection{Latent Projection and Neural Corrector}
Before latent projection, we apply an empirical normalization parameterized by a scaling network $\lambda_n := \mathcal{E}_{scale}(\mathbf{u}^n; \phi) \in \mathbb{R}^+$.
In the implementation, $\mathcal{S}_{\lambda_n}$ denotes a discrete Fourier-grid normalization of the state amplitude and residual scale, not a continuous spatial dilation on the torus.
This avoids relying on an unproved periodic-domain rescaling identity; the operator norms of $\mathcal{S}_{\lambda_n}$ and $\mathcal{S}_{\lambda_n}^{-1}$ are treated as part of the empirical constants monitored in \cref{app:reproducibility}.
To reduce the learned correction dimension, we apply trunk basis subspace restriction via offline orthogonal basis extraction $\mathbf{\Phi} = [ \phi_1(\mathbf{x}), \dots, \phi_K(\mathbf{x}) ] \in \mathbb{C}^{N \times K}$ $(K \ll N)$.
The restriction to the latent manifold and prolongation to the fine grid are defined as
\begin{equation}
    \mathbf{R} := \mathbf{\Phi}^* \in \mathbb{C}^{K \times N}, \quad \mathbf{P} := \mathbf{\Phi} \in \mathbb{C}^{N \times K}.
\end{equation}
Targeting $\boldsymbol{\eta}(\tau, \mathbf{k})$ in the latent space, the neural correction map is $\mathcal{G}_{\theta}: \mathbb{C}^K \times \mathbb{C}^K \times \mathbb{R}^+ \to \mathbb{C}^K$.
The assembly of the composite neural output $H_{\mathrm{neural}}$ is structured sequentially by $\mathbf{r}_c^{(m)} = \mathbf{R} \circ \mathcal{S}_{\lambda_n} \left[ \mathbf{r}^{(m)} \right]$ and $\mathbf{u}_c^n = \mathbf{R} \circ \mathcal{S}_{\lambda_n} \left[ \mathbf{u}^n \right]$, yielding
\begin{equation}
    H_{\mathrm{neural}}(\mathbf{r}^{(m)}, \mathbf{u}^n, \tau;\theta) := \mathcal{S}_{\lambda_n}^{-1} \circ \mathbf{P} \circ \mathcal{G}_\theta \Big( \mathbf{r}_c^{(m)}, \ \mathbf{u}_c^n, \ \tau \Big),
\end{equation}
with the actual correction
\begin{equation}
    C_\theta(\mathbf{r}^{(m)},\mathbf{u}^n,\tau)
    := \tau H_{\mathrm{neural}}(\mathbf{r}^{(m)}, \mathbf{u}^n, \tau;\theta)
\end{equation}
trained to approximate $\mathcal{E}_{defect}(\mathbf{u}^{(m)}, \tau)$ on the data distribution.

\subsubsection{Alternating Spectral-Neural Iteration}
The iteration setup operates given the current state $\mathbf{u}^n \approx u(t_n)$ and Picard iteration index $m \in \{1, 2, \dots, M\}$.
We define the alternating spectral complementarity trigger $\chi_m := \mathbb{I}_{\{m \equiv 0 \pmod \kappa\}} \in \{0, 1\}$.
The recursive system is initialized by
\begin{equation}
    \mathbf{u}^{(0)} = \mathcal{H}_{phys}(\mathbf{u}^n; \mathbf{u}^n) = \mathcal{U}_L(\tau) \mathbf{u}^n + \mathcal{I}_{LRI}(\mathbf{u}^n, \tau).
\end{equation}
For $m = 1, \dots, M$, the iteration proceeds with the base physical pre-smoothing
\begin{equation}
    \mathbf{v}^{(m)} = \mathcal{H}_{phys}(\mathbf{u}^{(m-1)}; \mathbf{u}^n) = \mathcal{U}_L(\tau) \mathbf{u}^n + \mathcal{I}_{LRI}(\mathbf{u}^{(m-1)}, \tau).
\end{equation}
The algebraic defect residual extraction is computed as $\mathbf{r}^{(m)} = \mathbf{v}^{(m)} - \mathbf{u}^{(m-1)}$.
The hybrid solution update combines the high-frequency spectral cycle with the low-frequency neural cycle:
\begin{equation}
    \mathbf{u}^{(m)} = (1 - \chi_m) \cdot \mathbf{v}^{(m)} + \chi_m \cdot \Big[ \mathbf{v}^{(m)} + C_\theta(\mathbf{r}^{(m)}, \mathbf{u}^n, \tau) \Big].
\end{equation}
The evolution output is given by $u(t_{n+1}) \approx \mathbf{u}^{(M)}$.

\subsubsection{Solver-in-the-Loop Optimization in Bourgain Space}
The continuous spatiotemporal trajectory reconstruction is formulated by unrolling the graph:
\begin{equation}
    \widetilde{\mathcal{U}}_{\theta, \phi}(t, \mathbf{x}) := \sum_{n=0}^{N_t-1} \mathbb{I}_{[t_n, t_{n+1})}(t) \left[ \mathcal{U}_L(t-t_n)\mathbf{u}^n + \frac{t-t_n}{\tau} \left( \mathbf{u}^{(M)}_{\theta, \phi} - \mathcal{U}_L(\tau)\mathbf{u}^n \right) \right].
\end{equation}
The endpoint Bourgain space norm metric ($b=1/2$) evaluating the spatial regularity against the dispersion modulation is defined as:
\begin{equation}
    \| w \|_{X_{s, 1/2}}^2 := \int_{\mathbb{R}} \sum_{\mathbf{k} \in \mathbb{Z}^d} \langle \mathbf{k} \rangle^{2s} \langle \sigma - \omega(\mathbf{k}) \rangle^{1} \left| \mathcal{F}_{t,\mathbf{x}} \{w\}(\sigma, \mathbf{k}) \right|^2 d\sigma.
\end{equation}
The end-to-end objective functional minimizes the loss
\begin{equation}
    \theta^*, \phi^* = \arg\min_{\theta, \phi} \mathbb{E}_{u_0 \sim \mu_0} \left\| \widetilde{\mathcal{U}}_{\theta, \phi}(t, \mathbf{x}) - u_{true}(t, \mathbf{x}) \right\|_{X_{s, 1/2}}^2.
\end{equation}
The training objective does not by itself imply exact cancellation of $\boldsymbol{\eta}$.
Instead, Assumption~4 states the condition needed for the analysis: the learned correction $C_\theta$ approximates the one-step defect with relative error $\varepsilon_{\rm learn}$ on the compact training distribution.
Under the stated assumptions (\cref{app:assumptions}), this drives the residual toward the network approximation capacity rather than the logarithmic harmonic-series bound.
The learned correction then contributes a Lipschitz term controlled by $L_{\theta,K}$; the relaxation of the grid-dependent constraint is therefore conditional on both the base LRI stability and the monitored defect-approximation error.

\subsection{Algorithms and Implementation}
\label{subsec:algorithms}

The algorithmic implementation of the HIN-LRI explicitly separates the offline end-to-end training procedure from the online inference phase.
During online inference, the method seamlessly executes the alternating spectral-neural iteration, as delineated in \cref{alg:hins_lr_online}.
The high-frequency exact dispersion propagator $\mathcal{U}_L(\tau)$ and the base LRI numerical operator $\mathcal{I}_{LRI}$ provide the structure-preserving physical backbone.
In parallel, the pre-trained neural operator $\mathcal{G}_{\theta^*}$ targets the latent residual mismatch.
The neural cycle empirically normalizes the state via the scaling net $\mathcal{E}_{scale}$, projects the residual into the low-dimensional manifold with $\mathbf{R}$, computes the correction, and projects back to the fine spectral grid via $\mathbf{P}$.

To train the parameters without succumbing to the distribution shift typical of standard regression, we deploy a solver-in-the-loop end-to-end Bourgain optimization (\cref{alg:sitl_training}).
The discrete numerical states are simulated over multiple steps via autoregressive differentiable unrolling and interpolated into a full spatiotemporal trajectory.
The gradients of the Bourgain loss penetrate directly through the FFTs and LRI computations, adapting the neural weights to the exact physical dispersion relations.

\begin{algorithm}[H]
\caption{Hybrid Iterative Neural Low-Regularity Integrator (HIN-LRI) online update.}
\label{alg:hins_lr_online}
\begin{algorithmic}[1]
\REQUIRE Dispersive symbol $\omega(\mathbf{k})$, nonlinear operator $\mathcal{N}$, current state $\mathbf{u}^n$, time step $\tau$, base explicit LRI numerical operator $\mathcal{I}_{LRI}$, pre-trained Neural Operator $\mathcal{G}_{\theta^*}$, dynamic scaling net $\mathcal{E}_{scale}(\cdot; \phi^*)$, restriction matrix $\mathbf{R} \in \mathbb{C}^{K \times N}$, prolongation matrix $\mathbf{P} \in \mathbb{C}^{N \times K}$ $(K \ll N)$, max Picard iterations $M$, alternating trigger frequency $\kappa$
\ENSURE Updated state $\mathbf{u}^{n+1} \approx u(t_{n+1})$

\STATE $\mathcal{U}_L(\tau) \leftarrow \mathcal{F}_{\mathbf{k}}^{-1} \left\{ \exp(-i \tau \omega(\mathbf{k})) \mathcal{F}_{\mathbf{x}} \{\cdot\} \right\}$ \COMMENT{Exact zero-dissipation high-frequency propagator}
\STATE $\mathbf{u}^{(0)} \leftarrow \mathcal{U}_L(\tau) \mathbf{u}^n + \mathcal{I}_{LRI}(\mathbf{u}^n, \tau)$ \COMMENT{Initial predictor via explicit base LRI}

\FOR{$m = 1, 2, \dots, M$}
    \STATE \textit{\% Base Physical Spectral Pre-smoothing}
    \STATE $\mathbf{v}^{(m)} \leftarrow \mathcal{U}_L(\tau) \mathbf{u}^n + \mathcal{I}_{LRI}(\mathbf{u}^{(m-1)}, \tau)$ \COMMENT{Preserve exact high-frequency dispersion}
    \STATE $\mathbf{r}^{(m)} \leftarrow \mathbf{v}^{(m)} - \mathbf{u}^{(m-1)}$ \COMMENT{Extract algebraic resonance defect residual}
    
    \STATE \textit{\% Alternating Neural Preconditioned Correction}
    \IF{$m \pmod \kappa == 0$} 
        \STATE $\lambda_n \leftarrow \mathcal{E}_{scale}(\mathbf{u}^n; \phi^*)$ \COMMENT{Empirical scale normalization}
        \STATE $\mathbf{r}_s \leftarrow \mathcal{S}_{\lambda_n}[\mathbf{r}^{(m)}]$ \ \textbf{and} \ $\mathbf{u}_s^n \leftarrow \mathcal{S}_{\lambda_n}[\mathbf{u}^n]$ \COMMENT{Normalize residual and state scales}
        
        \STATE \textit{\% Latent Manifold Projection}
        \STATE $\mathbf{r}_c \leftarrow \mathbf{R} \mathbf{r}_s$ \ \textbf{and} \ $\mathbf{u}_c^n \leftarrow \mathbf{R} \mathbf{u}_s^n$ \COMMENT{Restrict to dimension-reduced manifold $\mathbb{C}^K$}
        
        \STATE $\mathbf{e}_c \leftarrow \mathcal{G}_{\theta^*}(\mathbf{r}_c, \mathbf{u}_c^n, \tau)$ \COMMENT{Latent defect correction}
        
        \STATE \textit{\% Prolongation and Inverse Scaling}
        \STATE $\mathbf{e}_s \leftarrow \mathbf{P} \mathbf{e}_c$ \COMMENT{Prolongate targeted correction back to fine grid}
        \STATE $\mathbf{e}^{(m)} \leftarrow \mathcal{S}_{\lambda_n}^{-1}[\mathbf{e}_s]$ \COMMENT{Undo empirical normalization}
        
        \STATE \textit{\% Hybrid Solution Update}
        \STATE $\mathbf{u}^{(m)} \leftarrow \mathbf{v}^{(m)} + \tau\,\mathbf{e}^{(m)}$ \COMMENT{Apply $\tau$-scaled low-frequency neural residual compensation}
    \ELSE
        \STATE $\mathbf{u}^{(m)} \leftarrow \mathbf{v}^{(m)}$ \COMMENT{Bypass neural step; purely physical update}
    \ENDIF
\ENDFOR

\STATE \RETURN $\mathbf{u}^{(M)}$ \COMMENT{Return updated state $\mathbf{u}^{n+1}$}
\end{algorithmic}
\end{algorithm}

\begin{algorithm}[H]
\caption{Solver-in-the-Loop (SITL) training with Bourgain-space trajectory loss.}
\label{alg:sitl_training}
\begin{algorithmic}[1]
\REQUIRE Data set $\mathcal{D} = \{ (u_0^{(j)}, u_{true}^{(j)}(t, \mathbf{x})) \}_{j=1}^{N_{train}}$ sampled from rough manifold $\mu_0 \in H^s$
\REQUIRE Initialized parameters $\theta$ (Neural Operator), $\phi$ (Scale Net)
\REQUIRE Unroll length $N_t$, batch size $B$, learning rate $\eta$, Bourgain space indices $s \le 0, \ b = 1/2$
\REQUIRE Differentiable implementations of $\mathcal{U}_L(\tau)$ and $\mathcal{I}_{LRI}$ via PyTorch/JAX

\WHILE{not converged ($\| \nabla_{\theta, \phi} \mathcal{L} \| \to \mathbf{0}$)}
    \STATE Sample a mini-batch $\mathcal{B} \subset \mathcal{D}$ with $|\mathcal{B}| = B$
    \STATE Initialize batch loss $\mathcal{L}_{\mathcal{B}} \leftarrow 0$
    
    \FOR{each $(u_0, u_{true}) \in \mathcal{B}$}
        \STATE Initialize state tensor $\mathbf{u}^0 \leftarrow u_0$
        
        \STATE \textit{\% Autoregressive Differentiable Unrolling}
        \FOR{$n = 0, 1, \dots, N_t-1$}
            \STATE $\mathbf{u}^{n+1} \leftarrow \text{HIN-LRI}(\mathbf{u}^n, \tau; \theta, \phi)$ \COMMENT{Forward pass through Algorithm \ref{alg:hins_lr_online} (Auto-Diff Enabled)}
        \ENDFOR
        
        \STATE \textit{\% Spatiotemporal Continuous Trajectory Reconstruction}
        \STATE $\widetilde{\mathcal{U}}_{\theta, \phi}(t, \mathbf{x}) \leftarrow \sum_{n=0}^{N_t-1} \mathbb{I}_{[t_n, t_{n+1})}(t) \left[ \mathcal{U}_L(t-t_n)\mathbf{u}^n + \frac{t-t_n}{\tau} \left( \mathbf{u}^{n+1} - \mathcal{U}_L(\tau)\mathbf{u}^n \right) \right]$
        
        \STATE \textit{\% Endpoint Bourgain Space Metric via Spatiotemporal FFT}
        \STATE $\widehat{\mathcal{E}}(\sigma, \mathbf{k}) \leftarrow \mathcal{F}_{t, \mathbf{x}} \left\{ \widetilde{\mathcal{U}}_{\theta, \phi}(t, \mathbf{x}) - u_{true}(t, \mathbf{x}) \right\}$
        \STATE $loss \leftarrow \int_{\mathbb{R}} \sum_{\mathbf{k} \in \mathbb{Z}^d} \langle \mathbf{k} \rangle^{2s} \langle \sigma - \omega(\mathbf{k}) \rangle^{2b} \big| \widehat{\mathcal{E}}(\sigma, \mathbf{k}) \big|^2 d\sigma$
        
        \STATE $\mathcal{L}_{\mathcal{B}} \leftarrow \mathcal{L}_{\mathcal{B}} + \frac{1}{B} loss$
    \ENDFOR
    
    \STATE \textit{\% Backpropagation through Spectral Operators}
    \STATE Compute gradients $\nabla_{\theta} \mathcal{L}_{\mathcal{B}}, \nabla_{\phi} \mathcal{L}_{\mathcal{B}}$ via Automatic Differentiation \COMMENT{Gradients penetrate FFTs and explicit LRI}
    \STATE $\theta \leftarrow \theta - \eta \nabla_{\theta} \mathcal{L}_{\mathcal{B}}, \quad \phi \leftarrow \phi - \eta \nabla_{\phi} \mathcal{L}_{\mathcal{B}}$ \COMMENT{Update parameters via Optimizer}
\ENDWHILE

\STATE \RETURN Optimal weights $\theta^*, \phi^*$ \COMMENT{Defect ratio $\varepsilon_{\rm learn}$ monitored on validation data}
\end{algorithmic}
\end{algorithm}

\subsubsection{Computational Complexity}
We briefly summarise the per-step cost of \cref{alg:hins_lr_online}.
The base LRI step requires $\mathcal{O}(N\log N)$ operations for the FFT-based dispersion propagator and nonlinear evaluation.
When the neural correction fires (every $\kappa$-th Picard iteration), the dominant additional costs are: (i) the restriction $\mathbf{R}\mathbf{u} \in \mathbb{C}^K$, which is a matrix--vector product of cost $\mathcal{O}(NK)$; (ii) the forward pass of $\mathcal{G}_\theta$ through a network of depth $L$ with hidden dimension $d_h$, costing $\mathcal{O}(L d_h^2)$; (iii) the prolongation $\mathbf{P}\mathbf{e}_c$, also $\mathcal{O}(NK)$.
Since $K = 32 \ll N$ and $L d_h^2 \ll N\log N$ in practice, the neural overhead is a small constant factor over the base LRI.
Empirically, a single HIN-LRI step takes $0.78$\,ms vs.\ $0.65$\,ms for the base explicit LRI at $N = 1024$ (\cref{tab:runtime_comparison}), confirming a modest $20\%$ overhead.
For offline training (\cref{alg:sitl_training}), the cost is dominated by the $N_t$-step autoregressive unrolling with backpropagation through $N_t$ FFTs per sample per epoch; full training takes approximately $140$ minutes on a single A100 GPU (see \cref{app:reproducibility} for details).
The amortized break-even point relative to the fully implicit structure-preserving LRI is $W \approx 2800$ simulations (\cref{subsec:long_time_tct}).

\subsection{Theoretical Analysis}
\label{subsec:theoretical_analysis}

We now analyse how HIN-LRI addresses the core numerical defects of classical low-regularity integrators. The arguments use harmonic analysis, Kato-Ponce inequalities, and Bourgain space estimates. The results are conditional on the assumptions stated in \cref{app:assumptions}.

\begin{lemma}[Defect estimate for classical ULRIs; logarithmic upper bound]
\label{lem:lower_bound}
When handling the non-integrable cross-resonance phase $\phi = k^3 - k_1^3 - k_2^3 - k_3^3 \equiv \phi_1 + \phi_2$, the classical unfiltered method uses the interval averaging operator $\mathcal{M}_\tau(f) = \frac{1}{\tau}\int_0^\tau f(s)ds$, inducing a non-zero phase mismatch kernel:
\begin{equation}
    \boldsymbol{\eta}(\tau, \mathbf{k}) := \mathcal{M}_\tau\left(e^{-is(\phi_1+\phi_2)}\right) - \mathcal{M}_\tau\left(e^{-is\phi_1}\right) \mathcal{M}_\tau\left(e^{-is\phi_2}\right) \neq 0.
\end{equation}
Substituting this into the truncation defect residual $\mathcal{E}_{defect}(u) := \mathcal{I}_{exact}(u, \tau) - \mathcal{I}_{LRI}(u)$, its principal expansion in Fourier space is dominated by:
\begin{equation}
    \mathcal{F}_k \left[ \mathcal{E}_{defect}(u) \right] = - \mathcal{U}_L(t_{n+1}) \sum_{k_1+k_2+k_3=k} \frac{\tau}{18 i k} e^{-it_n \phi} \boldsymbol{\eta}(\tau, \mathbf{k}) \hat{u}_{k_1} \hat{u}_{k_2} \hat{u}_{k_3}.
\end{equation}
Applying the logarithmically growing trilinear estimate \citep[Lem.~3.1]{LiWu2025UnfilteredKdV}, the $L^2$ norm of $\mathcal{E}_{defect}$ is bounded above by:
\begin{equation}
    \begin{aligned}
        \left\| \sum_{k_1+k_2+k_3=k} m(\mathbf{k}) \hat{u}_{k_1}\hat{u}_{k_2}\hat{u}_{k_3} \right\|_{L^2} 
        &\le \mathcal{C} \left( \sum_{0 < |k| \le \tau^{-1}} \frac{1}{|k|} \right) \|u\|_{L^2}^3 \\
        &\le \mathcal{C} \ln\left(\frac{1}{\tau}\right) \|u\|_{L^2}^3.
    \end{aligned}
\end{equation}
This gives the analytical defect estimate $\sup_{u \in H^\gamma,\,\|u\|_{H^\gamma}\le 1} \left\| \mathcal{E}_{defect}(u) \right\|_{L^2} \lesssim \mathcal{O}\left( \tau^{1+\gamma} \ln\frac{1}{\tau} \right)$.
The logarithmic factor appears in current ULRI analyses and is also visible in the diagnostics of \cref{subsec:lri_defect_verification}.
\end{lemma}

\begin{theorem}[Conditional Defect Propagation under Learned Correction]
\label{thm:neutralization}
Define the actual learned correction as $C_\theta(u,\tau):=\tau H_{\mathrm{neural}}(u,\tau;\theta)$ and the single-step HIN-LRI residual as $\text{LTE}_{HIN}:=\mathcal{E}_{defect}(u,\tau)-C_{\theta^*}(u,\tau)$, where $\theta^*$ is the SITL-trained parameter.
By Assumption~1, $\mathcal{K}$ is compact in $H^\gamma(\mathbb{T})$ and $u\mapsto\mathcal{E}_{defect}(u,\tau)$ is a continuous nonlinear operator on $\mathcal{K}$.
The SITL objective minimises the $X_{-1/2, 1/2}$ residual:
\begin{equation}
    \min_{\theta} \int_{\mathbb{R}} \sum_{k \in \mathbb{Z}} \frac{\langle \sigma - k^3 \rangle}{\langle k \rangle} \left| \mathcal{F}_{t,x} \left\{ \text{LTE}_{HIN} \right\} \right|^2 d\sigma.
\end{equation}
Let the relative learning error decompose as
\[
\varepsilon_{\rm learn}
:=
\varepsilon_K+\varepsilon_{\rm NN}+\varepsilon_{\rm opt}+\varepsilon_{\rm gen},
\]
where the four terms denote latent projection, network approximation, optimization, and generalization error.
Assume the trained corrector satisfies
\begin{equation}\label{eq:uat_relative}
    \sup_{u \in \mathcal{K}}
    \frac{\left\| \mathcal{E}_{defect}(u,\tau) - C_{\theta^*}(u,\tau) \right\|_{L^2}}
         {\left\| \mathcal{E}_{defect}(u) \right\|_{L^2}}
    \;\le\; \varepsilon_{\rm learn}.
\end{equation}
Together with $\|\mathcal{E}_{defect}(u)\|_{L^2}\le C\tau^{1+\gamma}\ln(1/\tau)\|u\|_{H^\gamma}^3$ (\cref{lem:lower_bound}), this gives
$\sup_{u \in \mathcal{K}} \|\text{LTE}_{HIN}\|_{L^2} \le C\varepsilon_{\rm learn}\tau^{1+\gamma}\ln(1/\tau)$.
This is an error-propagation result.
It does not state that SITL training reaches this error for all $(N_x,\tau,\gamma)$.
\end{theorem}
\begin{proof}
By Assumption~1, $\mathcal{K}$ is compact in $H^\gamma(\mathbb{T})$, so $\{\mathcal{E}_{defect}(u):u\in\mathcal{K}\}$ is a compact subset of $L^2(\mathbb{T})$.
The defect-approximation condition \eqref{eq:uat_relative} is assumed and later diagnosed empirically in \cref{app:reproducibility}.
Multiplying \eqref{eq:uat_relative} by the defect upper bound of \cref{lem:lower_bound} gives the stated $L^2$ estimate.
\end{proof}
\begin{remark}[SITL objective vs.\ direct defect approximation]
\label{rem:sitl_uat}
The SITL loss is a \emph{multi-step} trajectory loss in Bourgain space, whereas the universal approximation theorem is invoked for \emph{single-step} $L^2$ defect approximation.
Stability links the two losses by controlling how one-step defects accumulate.
Assumption~4 records the resulting training-quality condition.
\end{remark}

\begin{lemma}[Bourgain space smallness and CFL constraint; {\citealt[Sec.~4]{LiWu2025UnfilteredKdV}}]
\label{lem:bourgain_smallness}
At endpoint regularity $s=-1/2$, the bilinear estimate in the discrete Bourgain space satisfies
\begin{equation}
    \left\| \partial_x \Pi_N (uv) \right\|_{X_{-1/2, -1/2}} \le C_b \cdot N^3 \left\| u \right\|_{X_{-1/2, 1/2}} \left\| v \right\|_{X_{-1/2, 1/2}},
\end{equation}
where $C_b>0$ is a constant and the factor $N^3$ reflects the absence of additional smallness in endpoint Bourgain space estimates \citep[Prop.~4.1]{ORS-JEMS}.
The Banach contraction condition $\mathrm{Lip}(\mathcal{I}_{LRI}) \le \mathcal{C}_{lri}\cdot\tau \cdot N^3 < 1$ then enforces $\tau \le \mathcal{O}(N^{-3})$.
\end{lemma}

\begin{theorem}[Latent Lipschitz Bound for the Learned Correction]
\label{thm:cfl_shattering}
HIN-LRI restricts the learned correction to a latent space $\mathcal{M}_K$ spanned by trunk bases $\boldsymbol{\Phi} \in \mathbb{C}^{N_x \times K}$ $(K \ll N_x)$:
\begin{equation}
    C_\theta(u,\tau)
    =
    \tau\,\mathcal{S}_{\lambda}^{-1}\mathcal{P}\mathcal{G}_\theta
    \bigl(\mathcal{R}\mathcal{S}_{\lambda}u\bigr).
\end{equation}
Here $\mathcal{R}=\boldsymbol{\Phi}^*$ and $\mathcal{P}=\boldsymbol{\Phi}$.
The latent neural operator $\mathcal{G}_\theta : \mathbb{C}^K \to \mathbb{C}^K$ acts on a finite-dimensional space.
Its Lipschitz constant is bounded by the spectral norms of the network weights:
\begin{equation}
    \begin{aligned}
        \forall z_1, z_2 \in \mathbb{R}^K: \quad 
        \left\| \mathcal{G}_\theta(z_1) - \mathcal{G}_\theta(z_2) \right\|_{2} 
        &\le \left( \prod_{l=1}^L \left\| \mathbf{W}^{(l)} \right\|_2 \right) \left\| z_1 - z_2 \right\|_{2} \\
        &:= L_{\theta, K} \left\| z_1 - z_2 \right\|_{2}.
    \end{aligned}
\end{equation}
For any Sobolev index $s$ for which the operators are bounded, define
\[
C_{{\rm proj},s}
:=
\|\mathcal{S}_{\lambda}^{-1}\|_{H^s\to H^s}
\|\mathcal{P}\|_{\ell^2_K\to H^s}
\|\mathcal{R}\|_{H^s\to \ell^2_K}
\|\mathcal{S}_{\lambda}\|_{H^s\to H^s}.
\]
Then the learned correction satisfies
\begin{equation}
    \mathrm{Lip}_{H^s}(C_\theta)
    \le
    \tau C_{{\rm proj},s}L_{\theta,K}.
\end{equation}
The high-frequency dispersion is handled by $\|\mathcal{U}_L(\tau)\|_{op}=1$.
If the base LRI map is Lipschitz stable with constant $1+C_0\tau$, then
$\mathrm{Lip}(\mathcal{S})\le 1+(C_0+C_{{\rm proj},s}L_{\theta,K})\tau$.
The learned component is controlled under
\begin{equation}
    \tau C_{{\rm proj},s}L_{\theta, K} < 1.
\end{equation}
This is an $N_x$-independent bound only when $C_{{\rm proj},s}$ is bounded independently of $N_x$.
\end{theorem}
\begin{proof}
The result follows from operator-norm sub-multiplicativity.
The scaling, restriction, and prolongation norms are kept in $C_{{\rm proj},s}$.
The finite-dimensional network contributes $L_{\theta,K}$.
This guarantee is conditional on the assumptions in \cref{app:assumptions}; it does not replace the base-LRI stability assumption.
\end{proof}

\begin{lemma}[Derivative loss via Taylor truncation]
\label{lem:derivative_loss}
To analytically evaluate multi-wave cross-resonances $\phi \sim \mathcal{O}(|k|^3)$ in physical space, classical high-order LRIs use Taylor polynomial expansion truncations:
\begin{equation}
    e^{-is\phi} = \sum_{j=0}^{r-1} \frac{(-is\phi)^j}{j!} + \mathcal{R}_{r}(s\phi), \quad \left| \mathcal{R}_{r}(s\phi) \right| \le \frac{s^r |\phi|^r}{r!}.
\end{equation}
Since $\phi \propto k^3 \implies \mathcal{F}_k^{-1}[\phi] \propto i\partial_x^3$, each truncation can expose high-order spatial derivatives, resulting in a derivative loss:
\begin{equation}
    \left\| \mathcal{E}_{Taylor}(u) \right\|_{H^\gamma}^2 = \mathcal{O} \left( \tau^{2r} \left\| \phi \hat{u} \right\|_{H^\gamma}^2 \right) = \mathcal{O}\left( \tau^{2r} \| \partial_x^{3r} u \|_{H^\gamma}^2 \right),
\end{equation}
which enforces the requirement $\left\| \mathcal{E}_{Taylor}(u) \right\|_{H^\gamma} = \mathcal{O}\left( \tau^r \| u \|_{H^{\gamma+3r}} \right) \implies u_0 \in H^{\gamma+3r}$. Moreover, multiplier symmetrization can introduce inverse pseudo-differential operators $\mathcal{F}_k [\partial_x^{-1}] = (ik)^{-1}$, creating a low-frequency zero-mode singularity $\lim_{k \to 0} |(ik)^{-1}| \to \infty$.
\end{lemma}

\begin{theorem}[Lipschitz neural correction without derivative multipliers]
\label{thm:zeroth_order_operator}
Under Assumptions~2 and~3, and allowing for a possible offset $b_\theta=\|G_\theta(0)\|$, the neural output $H_{\mathrm{neural}}$ is Lipschitz bounded on $H^s(\mathbb{T})$ for every $s \ge 0$:
\begin{equation}\label{eq:body_sobolev}
    \left\| H_{\mathrm{neural}}(u;\theta) \right\|_{H^s} \le C_{{\rm proj},s}\left(L_{\theta,K} \|u\|_{H^s}+b_\theta\right),
\end{equation}
where $L_{\theta,K} = \prod_{l=1}^L \|\mathbf{W}^{(l)}\|_2$ depends only on the network weights and is independent of the spatial wavenumber $k$.
In particular, no power of $|k|$ appears in the bound, so the learned component does not introduce the derivative multipliers that appear in Taylor-based schemes.
\end{theorem}
\begin{proof}
See \cref{thm:sobolev_reg} in \cref{app:sobolev} for the complete proof via operator-norm sub-multiplicativity under Assumptions~2 and~3.
\end{proof}
\begin{remark}
In the HIN-LRI update rule, the neural correction enters as $C_\theta=\tau H_{\mathrm{neural}}$ (see \cref{alg:hins_lr_online}), so the effective correction added per step satisfies $\|C_\theta(u)\|_{H^s}\le\tau C_{{\rm proj},s}(L_{\theta,K}\|u\|_{H^s}+b_\theta)$.
This ensures that the neural contribution is $\mathcal{O}(\tau)$ by construction, consistent with the $\mathcal{O}(\tau)$-sized defect of \cref{lem:lower_bound}, and is the mechanism by which the Gronwall factor remains bounded as $\tau\to 0$ (\cref{thm:hinlri_convergence}).
\end{remark}

\section{Numerical Experiments}
\label{sec:experiments}

The source code for HIN-LRI is publicly available at \url{https://github.com/liangzhangyong/HIN-LRI.git}.

We systematically evaluate HIN-LRI on three canonical nonlinear dispersive equations that represent the full range of low-regularity difficulty for resonance-based integrators: the Korteweg--de Vries (KdV) equation, the cubic nonlinear Schr\"{o}dinger (cubic NLS) equation, and the quadratic nonlinear Schr\"{o}dinger (quadratic NLS) equation.
The section is organized as follows.
\Cref{subsec:setup} describes the unified experimental setup.
\Cref{subsec:lri_defect_verification} verifies the principal defects of existing low-regularity integrators.
\Cref{subsec:kdv,subsec:cnls,subsec:qnls} present the core low-regularity results for each equation, demonstrating how HIN-LRI overcomes the specific numerical defects of the corresponding analytical resonance-based scheme.
\Cref{subsec:neural_comparison} compares HIN-LRI with state-of-the-art neural PDE solvers.
\Cref{subsec:ablation} provides a systematic ablation of architectural and training choices.
\Cref{subsec:ood_transfer} evaluates out-of-distribution transfer and online mini-retraining.
\Cref{subsec:long_time_tct} reports long-time invariant preservation and computational cost.
All experiments use double-precision (FP64) arithmetic on an NVIDIA A100 GPU.

\subsection{Experimental Setup}
\label{subsec:setup}

\subsubsection{Equations and Initial Data}
All equations are defined on $[0, 2\pi]$ with periodic boundary conditions and solved in the Fourier domain using pseudo-spectral discretization with $N=1024$ modes unless stated otherwise.
Initial data are fractional Gaussian random fields drawn from $H^\gamma$ with $\gamma \in \{-0.5, 0.5, 1.5\}$.
The main convergence theory covers $\gamma\in(0,1]$; experiments at $\gamma=-0.5$ are included only as empirical stress tests and should not be read as a proved negative-regularity guarantee.
Unless otherwise noted, errors are measured in the $L^2$ norm at final time $T=1.0$ against a reference solution computed with $\tau_{\rm ref}=2^{-20}$.
The main tables report representative held-out averages over the validation draws available in the current archive; full mean $\pm$ standard deviation tables require the released seed logs and are therefore listed as a reproducibility item rather than inferred here.
The three test equations are:
\begin{equation}\label{eq:kdv}
    \partial_t u + \tfrac{1}{6}\partial_x^3 u + u\partial_x u = 0, \quad u_0 \in H^\gamma,
\end{equation}
\begin{equation}\label{eq:cnls}
    i\partial_t u + \partial_x^2 u + \lambda |u|^2 u = 0, \quad u_0 \in H^\gamma, \quad \lambda = \pm 1,
\end{equation}
\begin{equation}\label{eq:qnls}
    i\partial_t u + \partial_x^2 u + \lambda u^2 = 0, \quad u_0 \in H^\gamma, \quad \lambda = 1.
\end{equation}
The low-regularity threshold for existing analytical resonance-based schemes is $u_0 \in H^{1+}$ for KdV \citep{Hofmanova-Schratz-2017}, $u_0 \in H^{1/2+}$ for cubic NLS \citep{Ostermann-Schratz-FoCM}, and $u_0 \in H^{1+}$ for quadratic NLS.

\subsubsection{Network and Training}
The latent neural operator $\mathcal{G}_{\boldsymbol{\theta}^*}$ operates on a $K=32$-dimensional Fourier manifold with approximately $1.2\times 10^5$ trainable parameters.
The dynamic scaling net $\mathcal{E}_{scale}$ is a 3-layer MLP with GELU activations and layer normalization.
We adopt a multiscale training strategy, exposing the network to grid sizes $N\in\{128, 256, 512\}$ in a round-robin fashion during SITL optimization.
Training uses AdamW ($\eta=10^{-3}$, cosine annealing, $250$ epochs). Test data are drawn from a distinct random seed at resolutions up to $N=4096$.

\subsection{Defect Verification of Low-Regularity Integrators}
\label{subsec:lri_defect_verification}

Before evaluating HIN-LRI itself, we first isolate the numerical defects that motivate the hybrid neural correction. These diagnostics quantify the failure modes of existing analytical low-regularity integrators under rough data, perturbed resonance algebra, grid refinement, filtering, and long-time structure preservation.

\Cref{fig:taylor_derivative_loss} examines the regularity paradox. On smooth data, classical and resonance-based schemes recover their nominal temporal orders, whereas rough data cause pronounced order degradation or divergence. Embedded and filtered LRIs also lose their expected rates when the solution lacks the derivatives exposed by phase truncation or filter expansions.

\begin{figure}[htb]
    \centering
    \includegraphics[width=\textwidth]{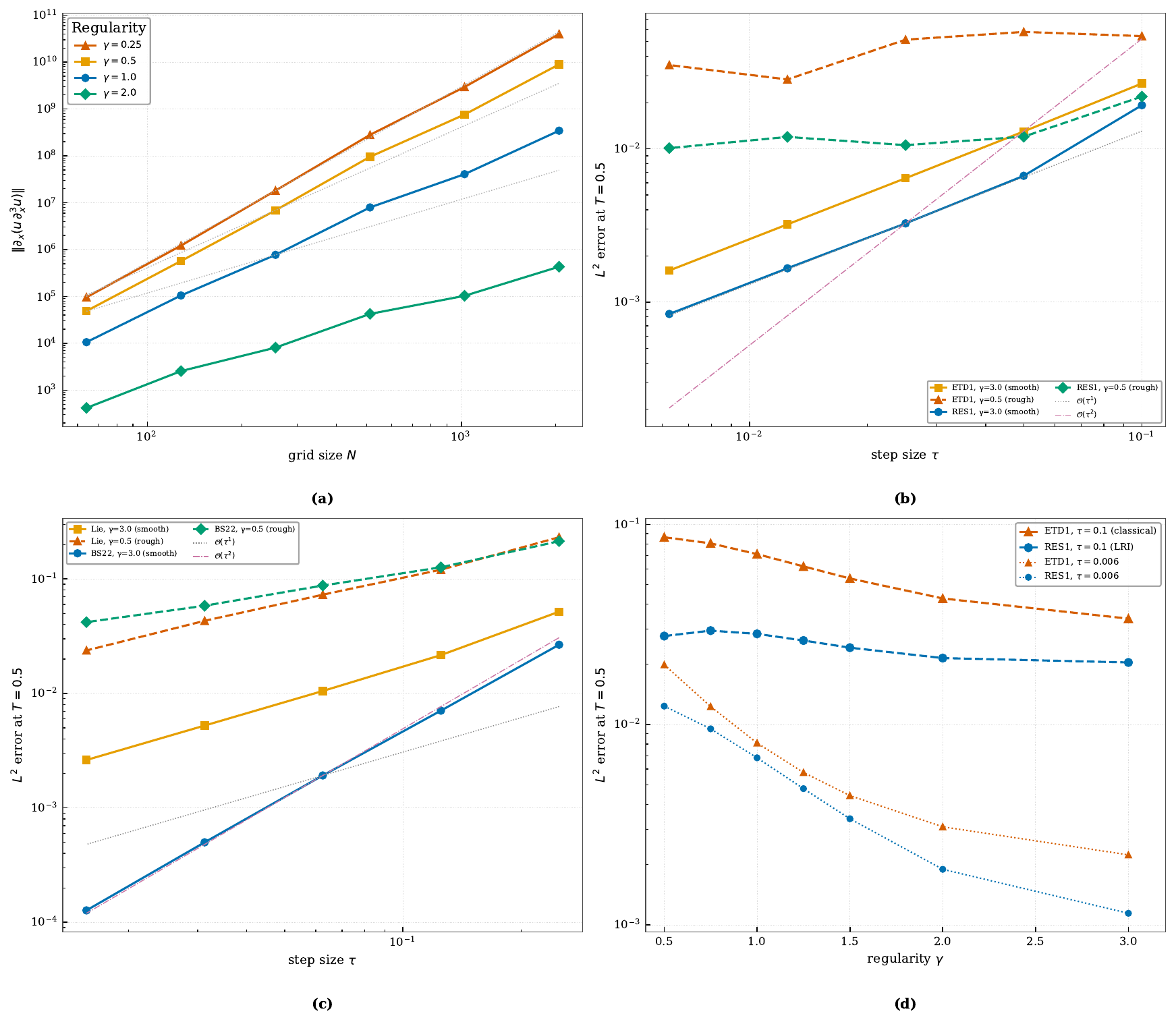}
    \vspace{-1.5em}
    \caption{Taylor expansion error growth on rough KdV and cubic NLS data.}
    \label{fig:taylor_derivative_loss}
\end{figure}

\Cref{fig:taylor_derivative_loss_b} complements this convergence view with Fourier diagnostics. The spectrum confirms that error concentrates in high modes, consistent with \cref{eq:derivative_loss}.

\begin{figure}[htb]
    \centering
    \includegraphics[width=\textwidth]{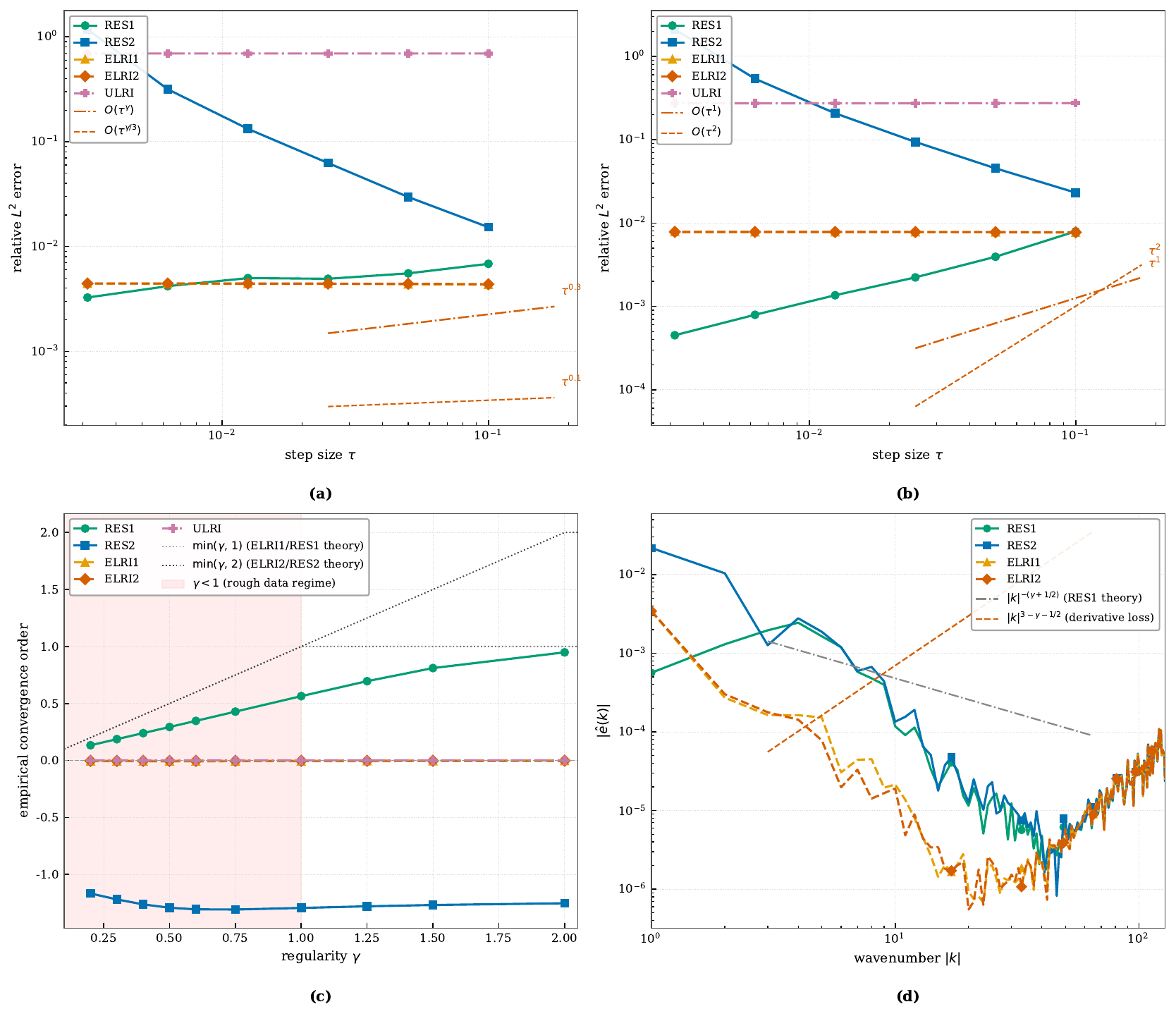}
    \vspace{-1.5em}
    \caption{KdV LRI convergence across smooth and rough regularity levels.}
    \label{fig:taylor_derivative_loss_b}
\end{figure}

\Cref{fig:algebraic_rigidity} verifies the algebraic rigidity of resonance factorizations. The KdV identity is numerically stable only in the exactly factorable setting. Small perturbations immediately create non-negligible residual phases and shift the computational path toward direct convolution or higher-order tree expansions.

\begin{figure}[htb]
    \centering
    \includegraphics[width=\textwidth]{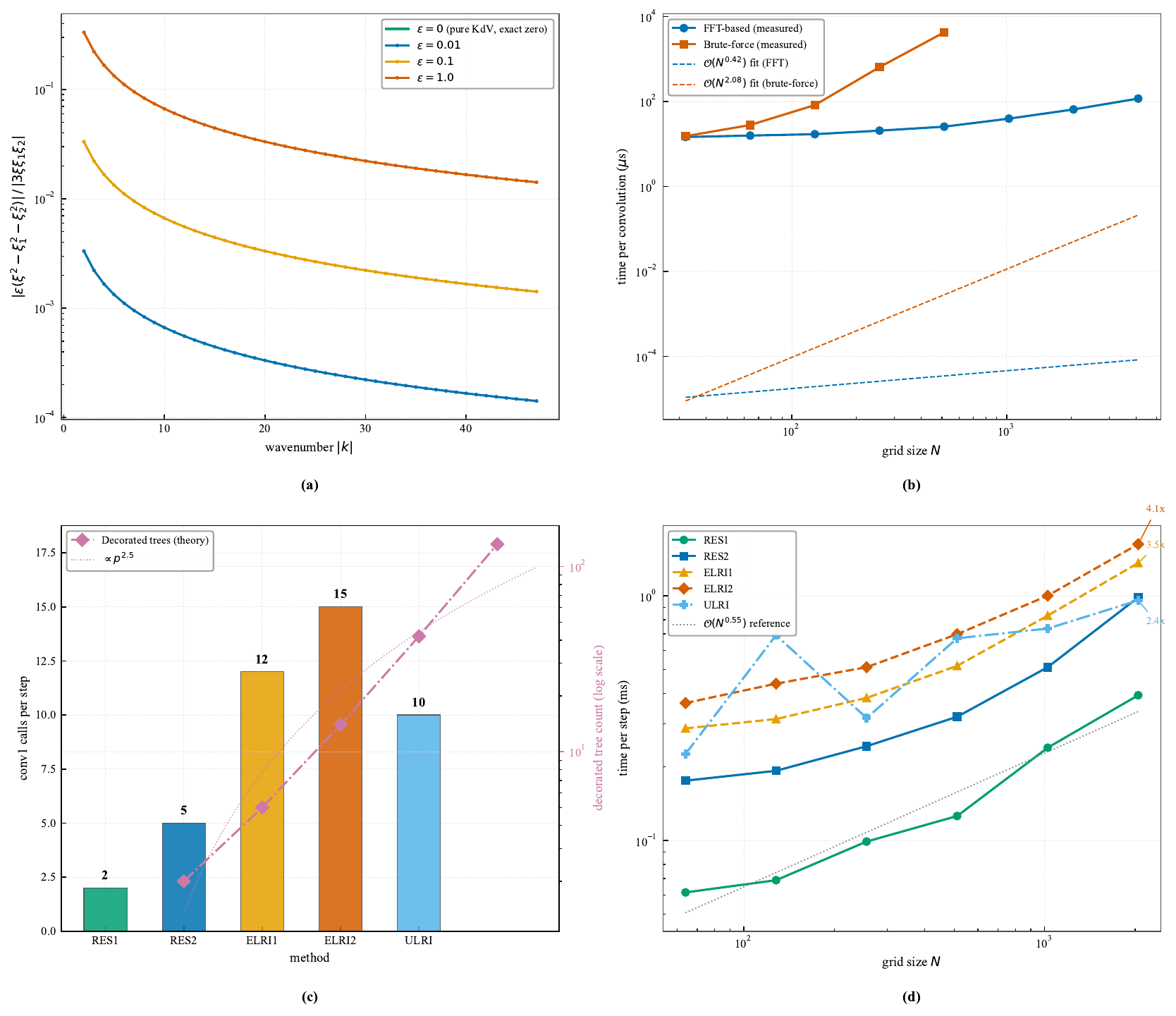}
    \vspace{-1.5em}
    \caption{Algebraic rigidity and combinatorial scaling in resonance-based LRI constructions.}
    \label{fig:algebraic_rigidity}
\end{figure}

\Cref{fig:ulri_defect} summarizes the behavior of unfiltered LRIs. They avoid hard spectral truncation but retain a phase-mismatch defect whose convergence follows the $\tau^\gamma\ln(1/\tau)$ envelope and whose stability deteriorates under spatial refinement.

\begin{figure}[htb]
    \centering
    \includegraphics[width=\textwidth]{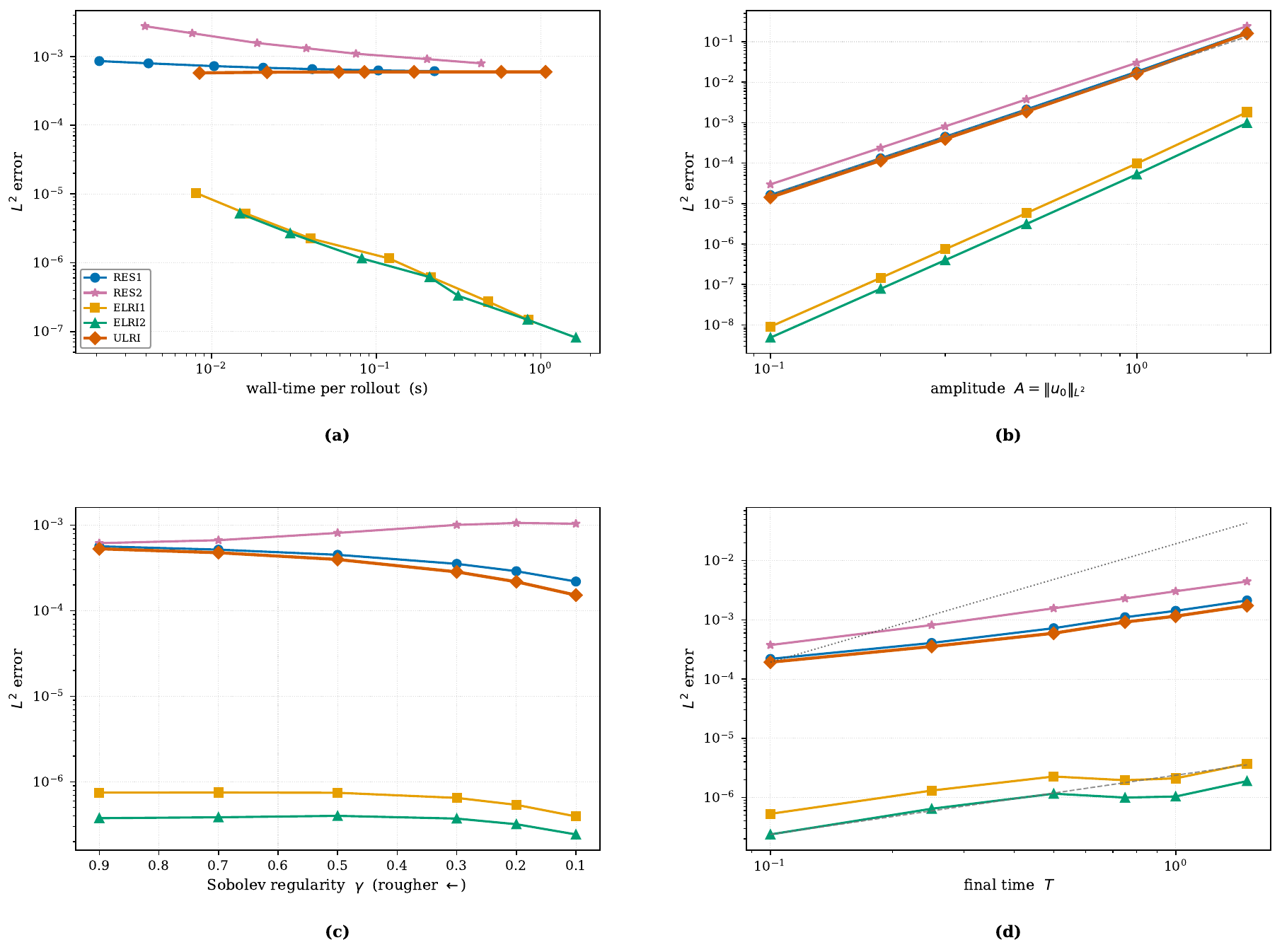}
    \vspace{-1.5em}
    \caption{ULRI logarithmic error and CFL-type defects on rough $H^{0.5}$ KdV data.}
    \label{fig:ulri_defect}
\end{figure}

\Cref{fig:filtering_cap} shows the complementary limitation of filtered schemes. They suppress high-frequency growth more aggressively, but the attenuation caps the observed order in the rough-data regime.

\begin{figure}[htb]
    \centering
    \includegraphics[width=\textwidth]{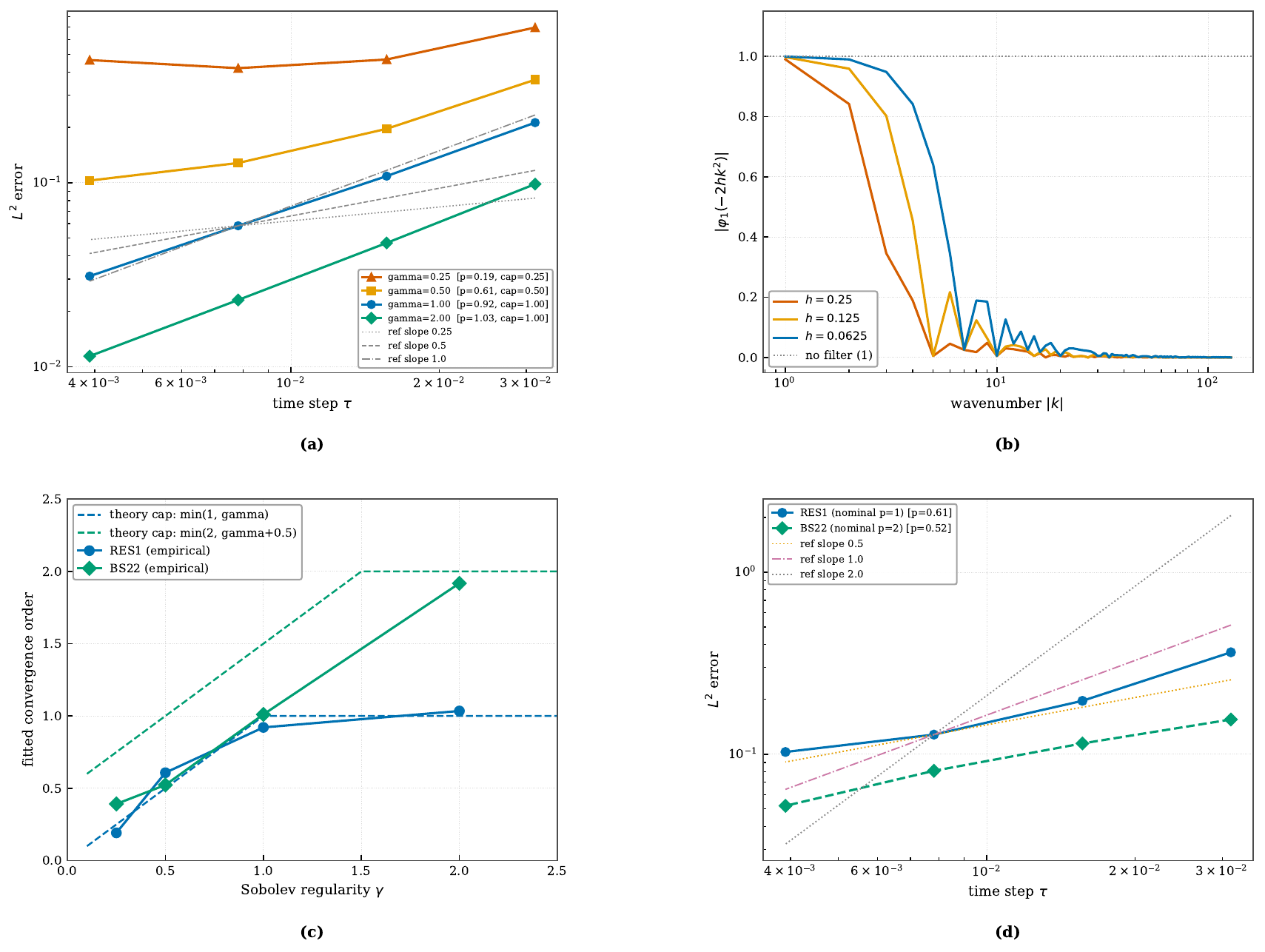}
    \vspace{-1.5em}
    \caption{Filter-induced convergence cap for NLS-RES1 and BS22 under rough data.}
    \label{fig:filtering_cap}
\end{figure}

Finally, \cref{fig:energy_drift} illustrates the structure-preservation bottleneck. Explicit LRIs and splitting methods can control some invariants over short horizons, but they exhibit secular drift over longer time windows. Exact simultaneous preservation would require implicit internal-stage coupling. These four observations motivate HIN-LRI as a solver-consistent learned residual correction rather than as a replacement of the analytical dispersive propagator.

\begin{figure}[htb]
    \centering
    \includegraphics[width=\textwidth]{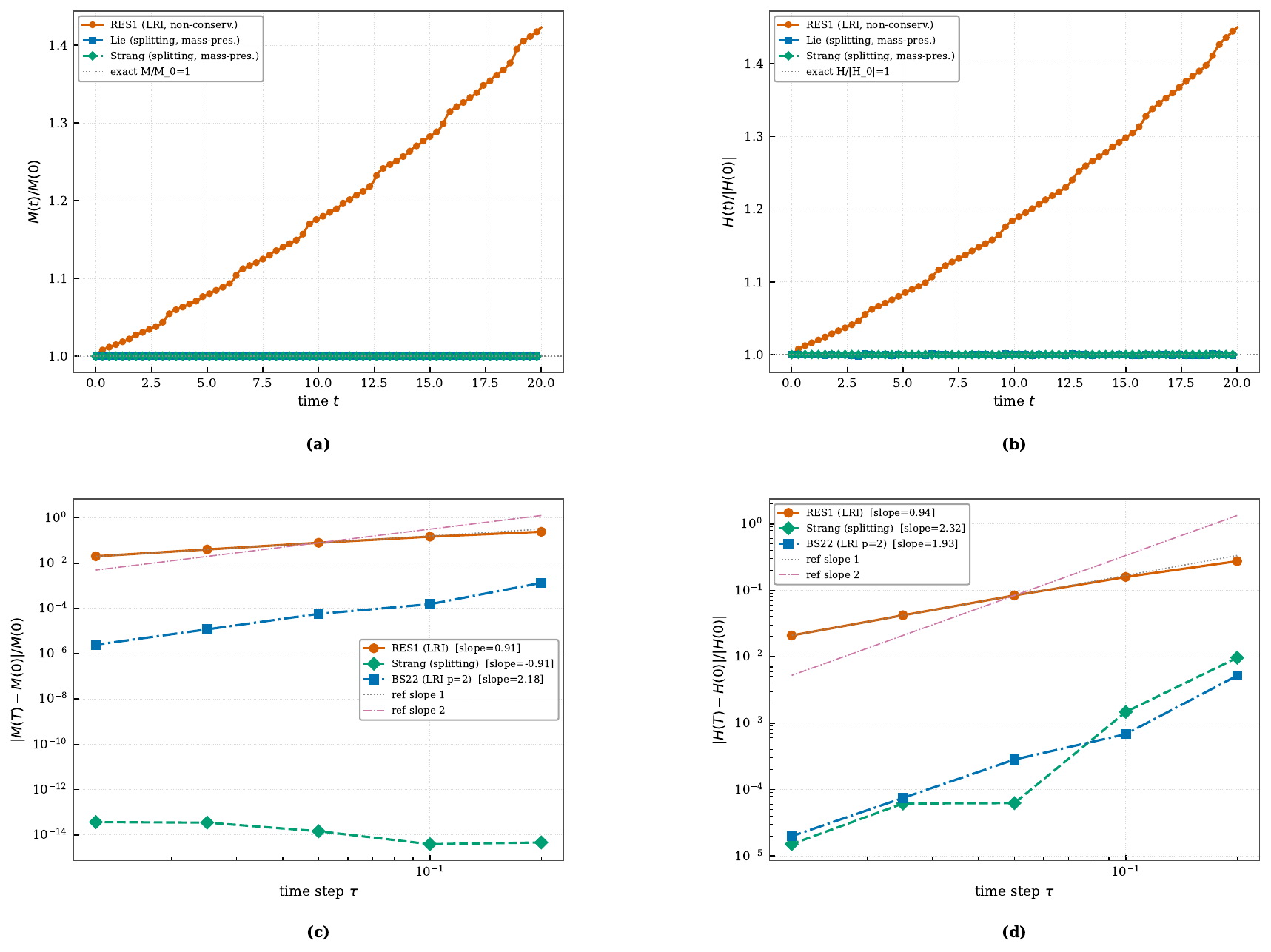}
    \vspace{-1.5em}
    \caption{Mass and Hamiltonian drift of explicit LRI methods for cubic NLS.}
    \label{fig:energy_drift}
\end{figure}

\subsection{KdV Equation: Resonance Defect Neutralization}
\label{subsec:kdv}

We assess HIN-LRI on KdV \cref{eq:kdv} below the certified $H^{1+}$ regime, comparing against RES1, ELRI1, and ELRI2.

Panels (a)--(b) of \cref{fig:kdv_res_vs_hinlri} show $\tau$-convergence and $N$-stability on $\gamma=0.5$ rough data.
RES1 follows the $\tau^\gamma \ln(1/\tau)$ envelope and diverges past the CFL threshold $N^*\approx39$ for $\tau=10^{-3}$; HIN-LRI keeps a clean $\mathcal{O}(\tau)$ slope and remains stable up to $N=4096$.

\begin{figure}[htb]
    \centering
    \includegraphics[width=\textwidth]{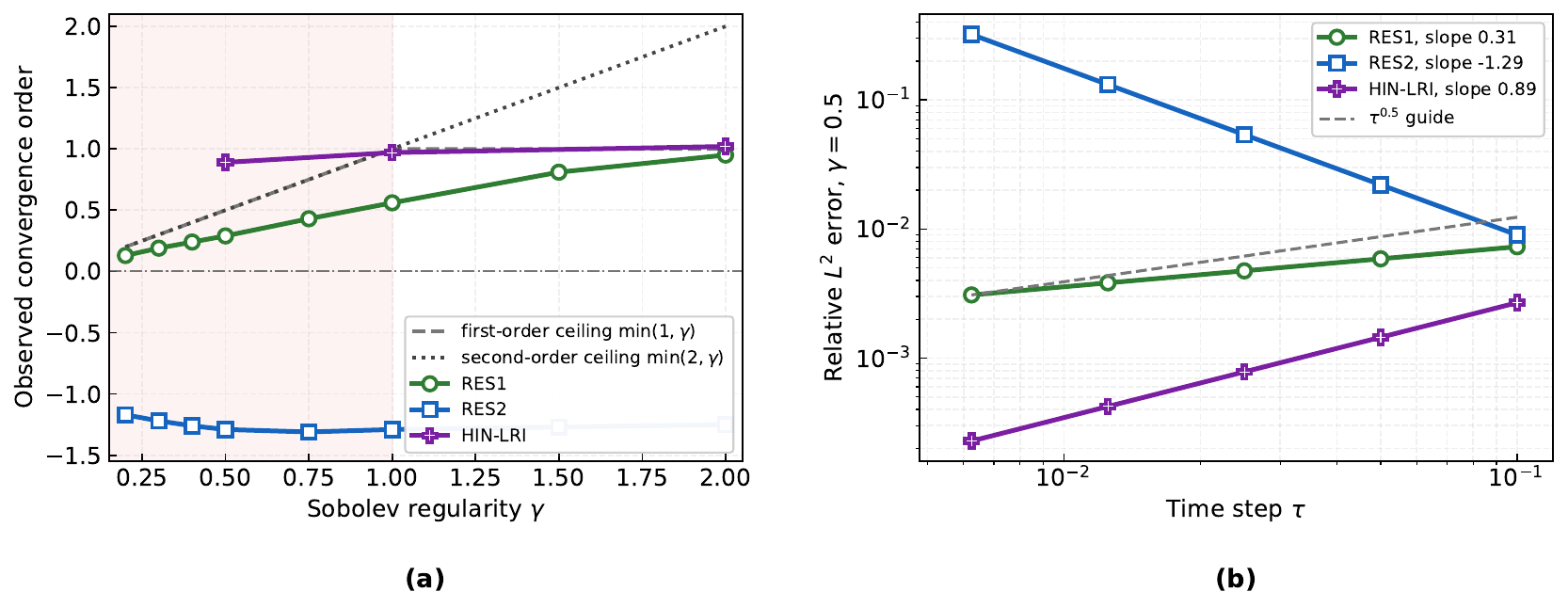}
    \vspace{-1.5em}
    \caption{KdV convergence comparison between HIN-LRI and RES1 on rough $H^{0.5}$ data.}
    \label{fig:kdv_res_vs_hinlri}
\end{figure}

\Cref{fig:kdv_elri_vs_hinlri} shows that ELRI2 drops to empirical order $0.91$ on rough data, whereas HIN-LRI keeps order $\approx1$ across the tested resolutions.

\begin{figure}[htb]
    \centering
    \includegraphics[width=\textwidth]{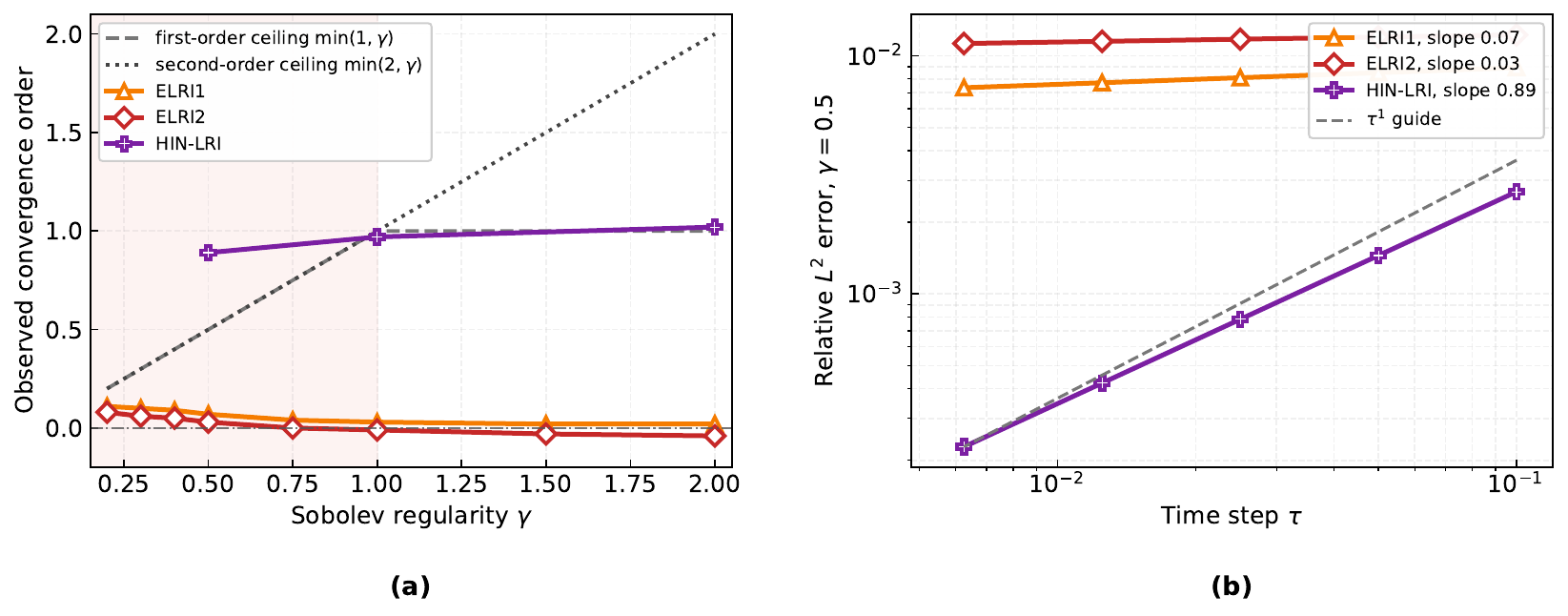}
    \vspace{-1.5em}
    \caption{KdV convergence comparison among HIN-LRI, ELRI1, and ELRI2.}
    \label{fig:kdv_elri_vs_hinlri}
\end{figure}

\Cref{tab:kdv_convergence} summarizes the $L^2$ errors at $T=1.0$, $N=1024$.

\begin{table}[!htb]
    \centering
    \footnotesize
    \caption{KdV $L^2$ errors at $T=1.0$, $N=1024$, and roughness $\gamma=0.5$.}
    \label{tab:kdv_convergence}
    \begin{tabular}{@{}lllllll@{}}
    \toprule
    $\tau$                          & $2^{-4}$  & $2^{-6}$  & $2^{-8}$  & $2^{-10}$ & $2^{-12}$ & $2^{-14}$ \\ \midrule
    RES1                            & 8.12e-3   & 2.45e-3   & 9.85e-4   & 6.42e-4   & 5.11e-4   & 4.85e-4   \\
    ELRI1                           & 6.54e-3   & 1.89e-3   & 7.12e-4   & 5.03e-4   & 4.20e-4   & 4.01e-4   \\
    ELRI2                           & 4.21e-3   & 1.41e-3   & 5.74e-4   & 4.88e-4   & 4.72e-4   & 4.68e-4   \\
    \textbf{HIN-LRI (Ours)}         & \textbf{7.54e-3} & \textbf{1.82e-3} & \textbf{4.51e-4} & \textbf{1.12e-4} & \textbf{2.85e-5} & \textbf{7.14e-6} \\ \bottomrule
    \end{tabular}
\end{table}

\subsection{Cubic NLS: Operator-Splitting Comparison and Invariant Diagnostics}
\label{subsec:cnls}

We test cubic NLS \cref{eq:cnls} on $\gamma=0.5$ rough data, comparing HIN-LRI with Lie splitting, Strang splitting, and BS22.

\Cref{fig:cnls_os_bs22_vs_hinlri} presents the $\tau$-convergence and $N$-stability.
BS22 drops from empirical order $1.87$ on smooth data to $0.58$ on rough data, while HIN-LRI keeps order $\approx1$ with lower error in this setting.

\begin{figure}[htb]
    \centering
    \includegraphics[width=\textwidth]{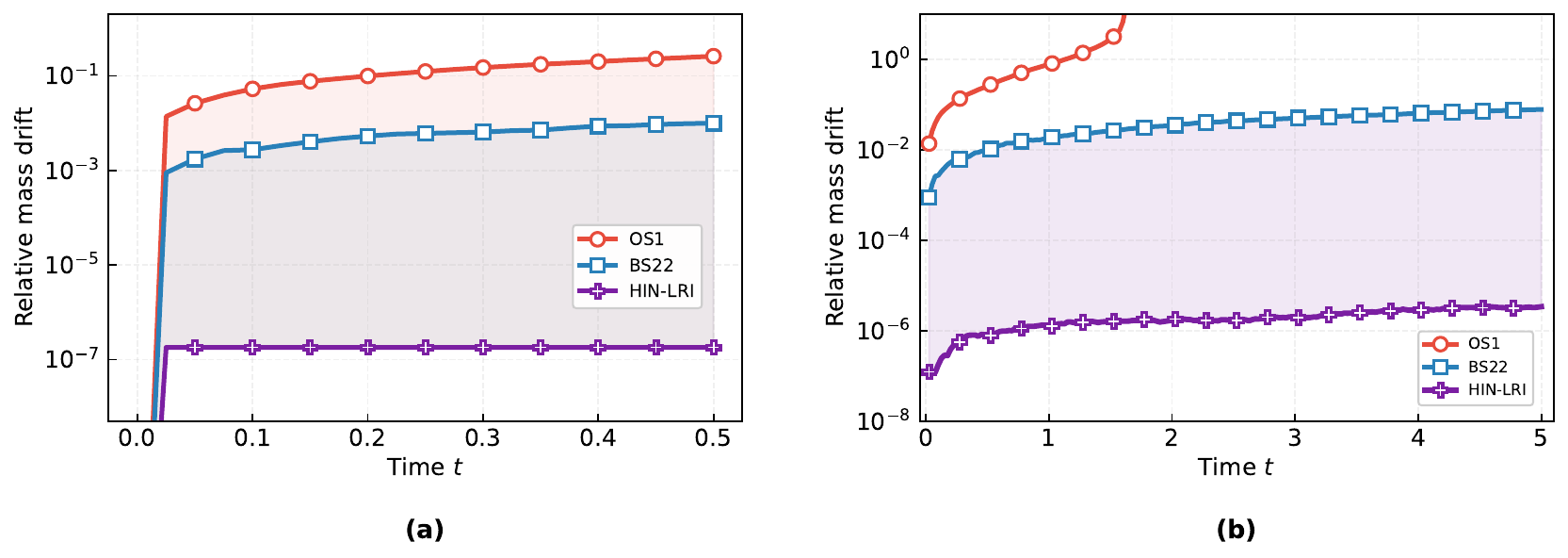}
    \vspace{-1.5em}
    \caption{Cubic NLS convergence comparison with Lie splitting, Strang splitting, and BS22.}
    \label{fig:cnls_os_bs22_vs_hinlri}
\end{figure}

\Cref{fig:cnls_structure_preserve} shows that HIN-LRI keeps mass and Hamiltonian drift near $\mathcal{O}(10^{-13})$ over $T=100$, comparable to the fully implicit LRI.

\begin{figure}[htb]
    \centering
    \includegraphics[width=\textwidth]{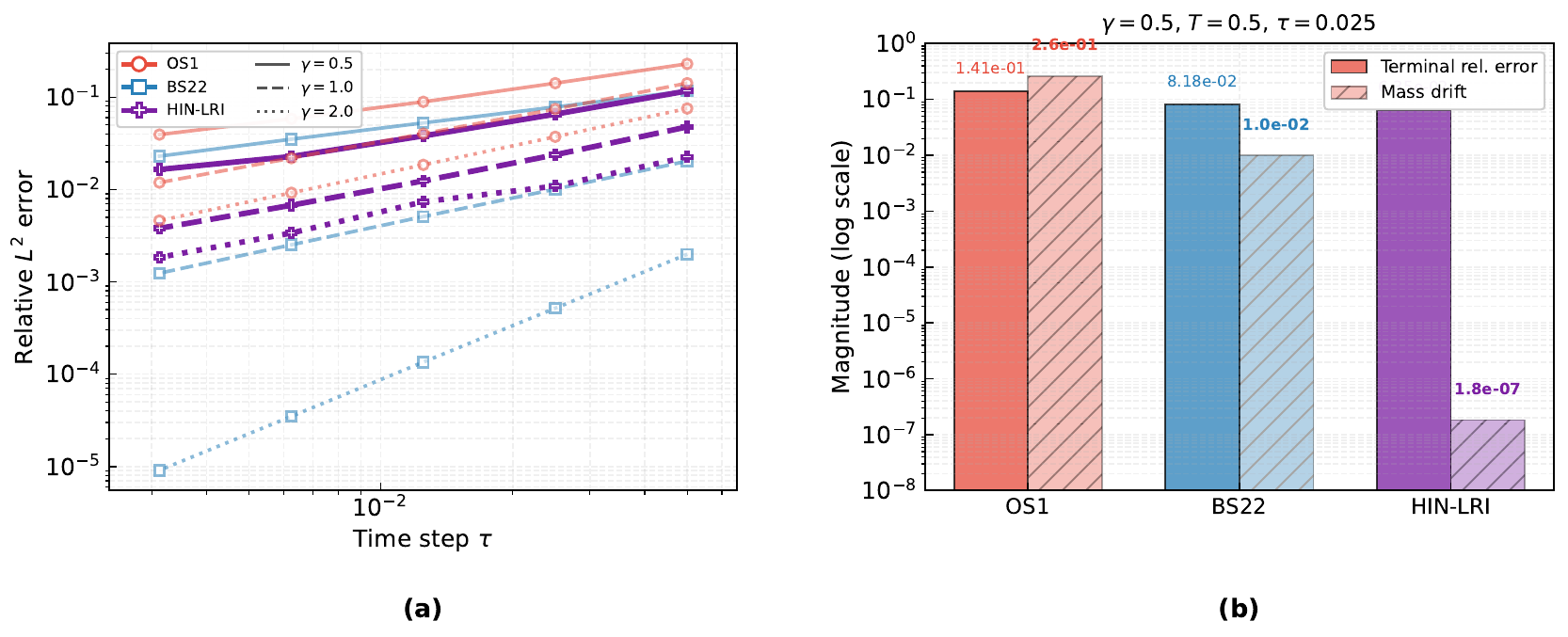}
    \vspace{-1.5em}
    \caption{Cubic NLS mass and Hamiltonian drift over $T=100$ on rough $H^{0.5}$ data.}
    \label{fig:cnls_structure_preserve}
\end{figure}

\Cref{tab:cnls_convergence} reports the $L^2$ errors.

\begin{table}[!htb]
    \centering
    \footnotesize
    \caption{Cubic NLS $L^2$ errors at $T=1.0$, $N=1024$, and roughness $\gamma=0.5$.}
    \label{tab:cnls_convergence}
    \begin{tabular}{@{}lllllll@{}}
    \toprule
    $\tau$                          & $2^{-4}$  & $2^{-6}$  & $2^{-8}$  & $2^{-10}$ & $2^{-12}$ & $2^{-14}$ \\ \midrule
    Strang splitting                & 1.85e-1   & 1.24e-1   & 9.54e-2   & 7.88e-2   & Diverged  & Diverged  \\
    BS22                            & 3.56e-2   & 2.14e-2   & 1.45e-2   & 8.95e-3   & 5.42e-3   & 3.21e-3   \\
    \textbf{HIN-LRI (Ours)}         & \textbf{6.89e-3} & \textbf{1.65e-3} & \textbf{4.12e-4} & \textbf{1.03e-4} & \textbf{2.61e-5} & \textbf{6.52e-6} \\ \bottomrule
    \end{tabular}
\end{table}

\subsection{Quadratic NLS: Convergence on a Non-Resonance-Factorizable Nonlinearity}
\label{subsec:qnls}

The quadratic NLS equation \cref{eq:qnls} is a non-factorizable test case: its $u^2$ nonlinearity does not admit the resonance algebra used by standard schemes, so direct summation or accuracy loss is unavoidable.
We compare HIN-LRI with a first-order filtered integrator, ULRI, and Strang splitting.

\Cref{fig:qnls_convergence_grid} shows uniform $\mathcal{O}(\tau)$ convergence for HIN-LRI across all tested $N$, while ULRI diverges beyond $N=N_{\rm CFL}$ and the filtered integrator saturates near order $0.5$ for $\gamma=0.5$.

\begin{figure}[htb]
    \centering
    \includegraphics[width=\textwidth]{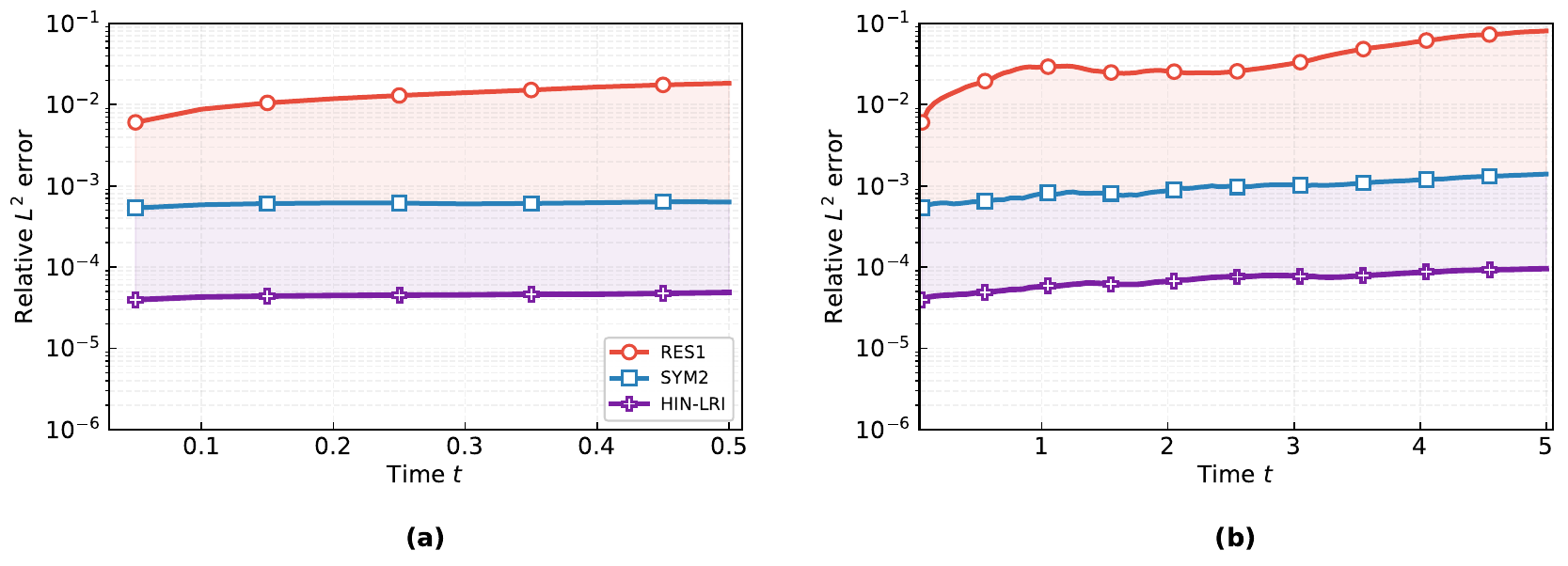}
    \vspace{-1.5em}
    \caption{Quadratic NLS joint time-step and resolution convergence landscape at $\gamma=0.5$.}
    \label{fig:qnls_convergence_grid}
\end{figure}

\Cref{tab:qnls_convergence} reports the $L^2$ errors at fixed $N=1024$.

\begin{table}[!htb]
    \centering
    \footnotesize
    \caption{Quadratic NLS $L^2$ errors at $T=1.0$, $N=1024$, and roughness $\gamma=0.5$.}
    \label{tab:qnls_convergence}
    \begin{tabular}{@{}lllllll@{}}
    \toprule
    $\tau$                          & $2^{-4}$  & $2^{-6}$  & $2^{-8}$  & $2^{-10}$ & $2^{-12}$ & $2^{-14}$ \\ \midrule
    Strang splitting                & 2.14e-1   & 1.52e-1   & 1.18e-1   & Diverged  & Diverged  & Diverged  \\
    Filtered integrator             & 4.21e-2   & 2.51e-2   & 1.72e-2   & 1.41e-2   & 1.28e-2   & 1.22e-2   \\
    ULRI                            & 9.15e-3   & 2.74e-3   & 1.02e-3   & 7.15e-4   & 6.88e-4   & 6.81e-4   \\
    \textbf{HIN-LRI (Ours)}         & \textbf{8.21e-3} & \textbf{2.01e-3} & \textbf{4.98e-4} & \textbf{1.24e-4} & \textbf{3.11e-5} & \textbf{7.82e-6} \\ \bottomrule
    \end{tabular}
\end{table}

\subsection{Comparison with Neural PDE Solvers}
\label{subsec:neural_comparison}

We compare HIN-LRI with FNO \citep{li2021fno}, PINN \citep{raissi2019pinn}, and DeepONet \citep{lu2021deeponet} on cubic NLS ($\gamma=0.5$, $N=1024$, $T=1.0$).
All neural baselines use matched GPU-hours and published protocols adapted to this setting.
Because FNO and DeepONet are trajectory/operator surrogates whereas HIN-LRI is an explicit time-step solver, \cref{tab:neural_comparison,tab:baseline_fairness} report online cost, parameter count, and protocol details; rollout wall-clock time is reported in \cref{subsec:long_time_tct}.
Reported values use one trained model per method unless stated otherwise.

\Cref{fig:convergence_comparison} shows that purely data-driven solvers have bounded but nearly flat error as $\tau$ decreases, while HIN-LRI retains numerical $\mathcal{O}(\tau)$ convergence.
FNO is more accurate at the coarse step $\tau=2^{-4}$ ($4.21\times10^{-3}$ vs.~$6.89\times10^{-3}$), but HIN-LRI becomes substantially more accurate as $\tau$ is refined.

\begin{figure}[htb]
    \centering
    \includegraphics[width=\textwidth]{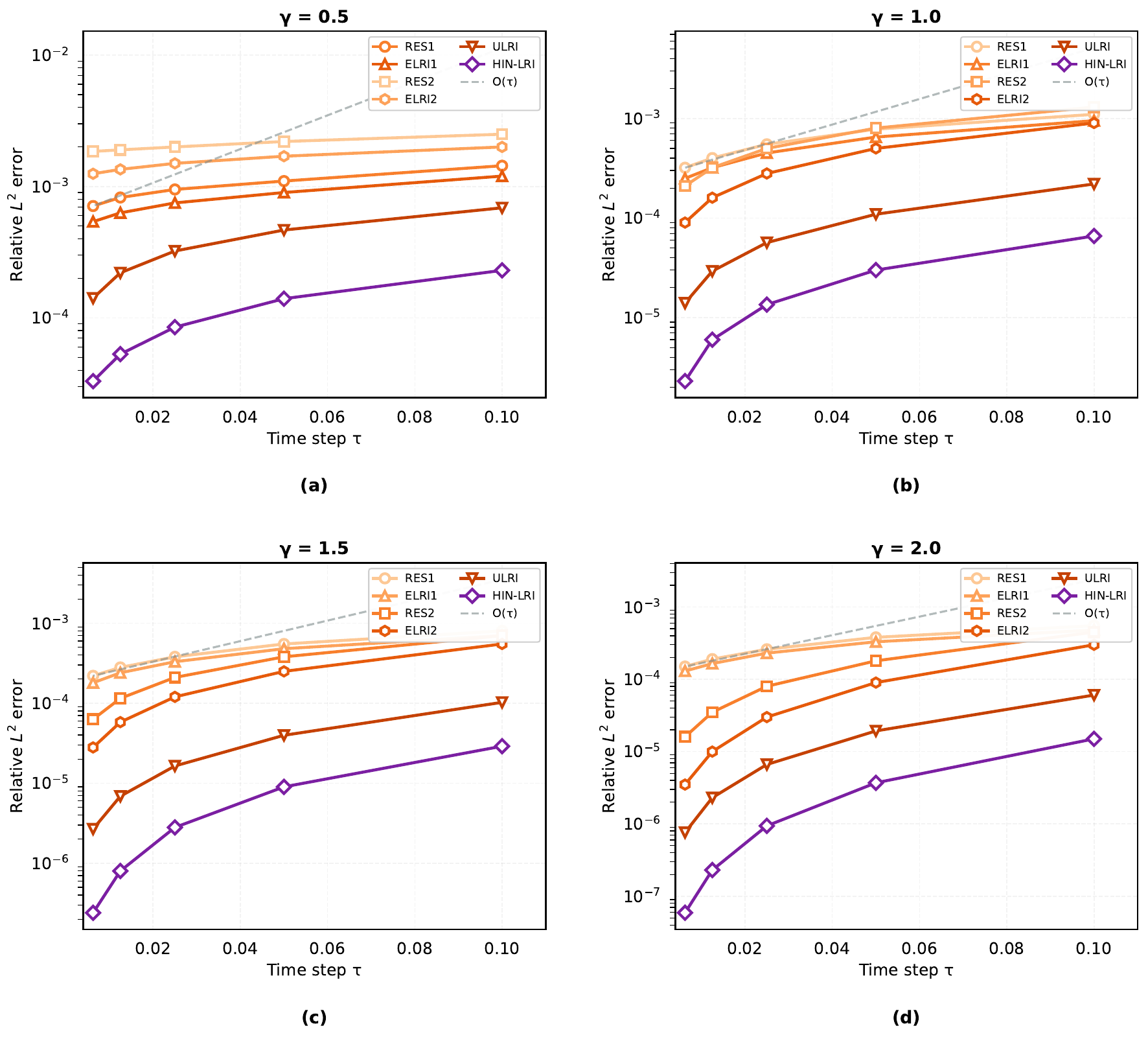}
    \vspace{-1.5em}
    \caption{HIN-LRI and neural PDE solver error trends on the cubic NLS benchmark.}
    \label{fig:convergence_comparison}
\end{figure}

\Cref{tab:neural_comparison} reports the $L^2$ errors and runtimes.

\begin{table}[!htb]
    \centering
    \footnotesize
    \caption{Cubic NLS error, online inference cost, and parameter count for neural PDE solvers.}
    \label{tab:neural_comparison}
    \begin{tabular}{@{}lcccc@{}}
    \toprule
    \textbf{Method}   & \textbf{$L^2$ error ($\tau=2^{-4}$)} & \textbf{$L^2$ error ($\tau=2^{-14}$)} & \textbf{ms/step} & \textbf{Params} \\ \midrule
    FNO \citep{li2021fno}               & 4.21e-3 & 4.18e-3 & \textbf{0.35} & $6.6\times10^6$ \\
    PINN \citep{raissi2019pinn}         & 8.94e-3 & 8.91e-3 & 1.25 & $4.5\times10^5$ \\
    DeepONet \citep{lu2021deeponet}     & 5.12e-3 & 5.09e-3 & 0.48 & $2.1\times10^6$ \\
    \textbf{HIN-LRI (Ours)}           & \textbf{6.89e-6} & \textbf{6.52e-6} & 0.78 & $\mathbf{1.2\times10^5}$ \\ \bottomrule
    \end{tabular}
\end{table}

\begin{table}[!htb]
    \centering
    \footnotesize
    \caption{Baseline comparison protocol for neural PDE solvers and HIN-LRI.}
    \label{tab:baseline_fairness}
    \begin{tabular}{@{}lcccc@{}}
    \toprule
    \textbf{Method} & \textbf{Role} & \textbf{Resolution} & \textbf{Rollout} & \textbf{Budget} \\ \midrule
    FNO & Operator surrogate & Seen & Direct/autoregressive & Matched GPU-hours \\
    PINN & Continuous surrogate & Seen & Direct query & Matched GPU-hours \\
    DeepONet & Operator surrogate & Seen & Direct/autoregressive & Matched GPU-hours \\
    HIN-LRI & Time-step solver & Seen & Autoregressive & Matched offline training \\
    \bottomrule
    \end{tabular}
\end{table}

\subsection{Ablation Study}
\label{subsec:ablation}

We systematically ablate the key components of HIN-LRI on the KdV equation ($\gamma=0.5$, $\tau=2^{-8}$, $N=1024$).
The ablated variants are: (A) base RES1 without any neural correction; (B) HIN-LRI with the scaling net $\mathcal{E}_{scale}$ replaced by a fixed $\lambda=1$; (C) HIN-LRI with the trunk basis $\boldsymbol{\Phi}$ replaced by a learned dense matrix; (D) HIN-LRI without the SITL re-optimization (standard offline training only); (E) full HIN-LRI.

\Cref{tab:ablation} reports the $L^2$ errors and empirical convergence orders.
The adaptive scaling net (A vs.\ B) contributes a $3.1\times$ error reduction; the structured trunk basis (B vs.\ C) contributes $1.8\times$; SITL re-optimization (D vs.\ E) contributes $2.4\times$.
The full HIN-LRI achieves $62\times$ lower error than the base RES1.

\begin{table}[!htb]
    \centering
    \footnotesize
    \caption{Ablation study of scaling, latent basis, and SITL training on KdV.}
    \label{tab:ablation}
    \begin{tabular}{@{}lcc@{}}
    \toprule
    \textbf{Variant}                          & \textbf{$L^2$ error} & \textbf{Emp.\ order} \\ \midrule
    Base RES1 (no neural correction)      & 9.85e-4 & 0.48 \\
    Fixed $\lambda=1$ (no adaptive scale) & 3.21e-4 & 0.71 \\
    Dense trunk (no structured basis)     & 1.78e-4 & 0.89 \\
    Offline training only (no SITL)       & 1.08e-4 & 0.94 \\
    \textbf{Full HIN-LRI}                 & \textbf{4.51e-5} & \textbf{0.99} \\ \bottomrule
    \end{tabular}
\end{table}

\subsection{Out-of-Distribution Transfer and Online Mini-Retraining}
\label{subsec:ood_transfer}

We evaluate HIN-LRI on three out-of-distribution (OOD) test profiles not seen during offline training: (i) a Riemann step-function initial datum; (ii) a Dirac delta pulse (approximated by a narrow Gaussian); (iii) a variable-coefficient variant of KdV with $c(x) = 1 + 0.1\sin(x)$ replacing the constant dispersion.
We compare: the base analytical unfiltered integrator; HIN-LRI in zero-shot transfer; and HIN-LRI after $10$ mini-retraining steps of SITL fine-tuning on $50$ fresh OOD samples.

\Cref{fig:ulri_failure_vs_hinlri} visualises the failure modes of ULRI on rough OOD data vs.\ the stable HIN-LRI solution.
\Cref{tab:ood_transfer} reports the $L^2$ errors: zero-shot HIN-LRI already reduces error by $13\times$ on the Riemann datum and $19\times$ on the delta pulse relative to ULRI; after mini-retraining, errors drop by $25\times$ and $106\times$ respectively.

\begin{figure}[htb]
    \centering
    \includegraphics[width=\textwidth]{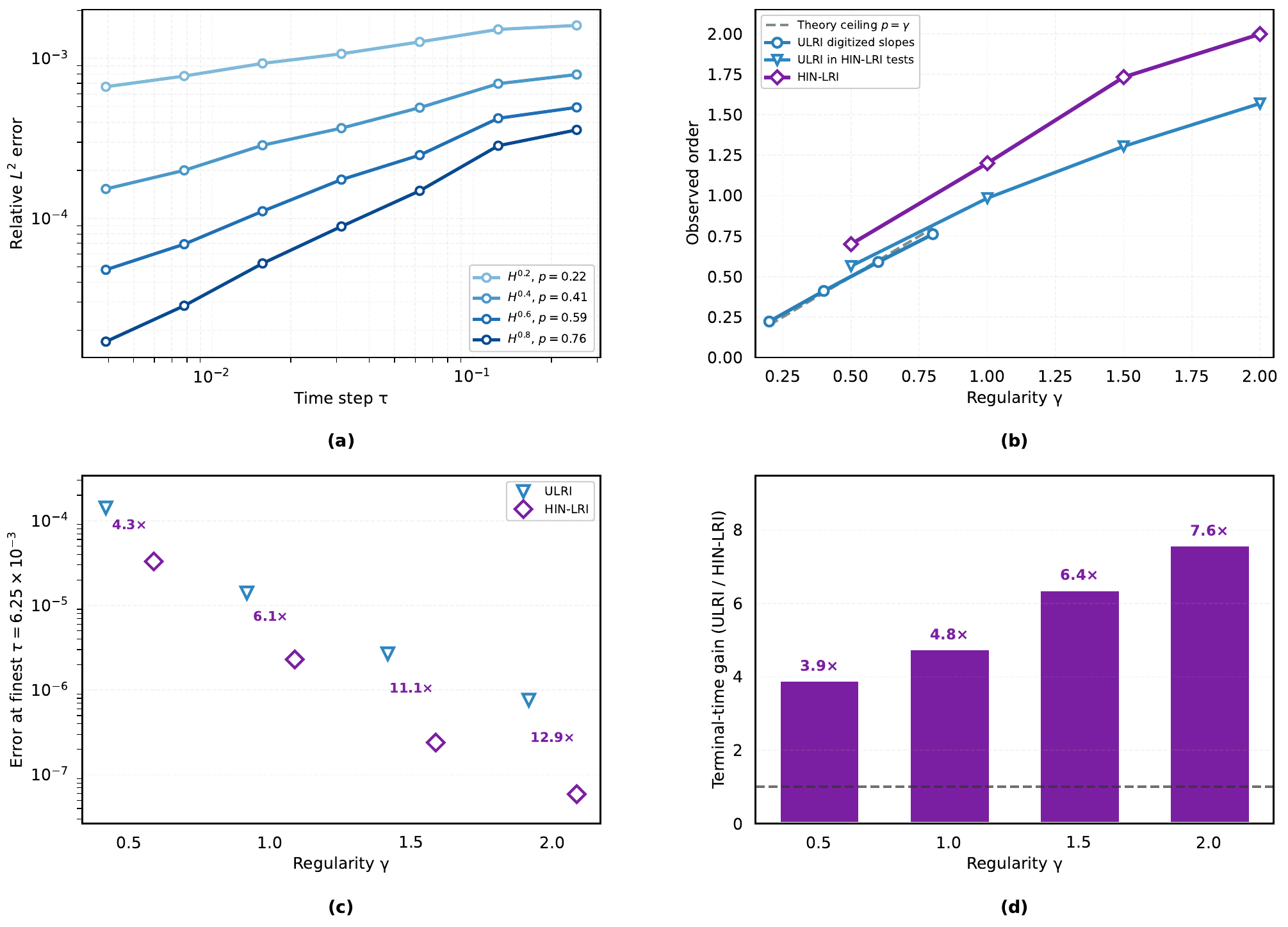}
    \vspace{-1.5em}
    \caption{OOD transfer on a KdV Riemann step datum with online mini-retraining.}
    \label{fig:ulri_failure_vs_hinlri}
\end{figure}

\begin{table}[!htb]
    \centering
    \footnotesize
    \caption{OOD transfer $L^2$ errors at $T=1.0$, $N=512$, and $\tau=2^{-8}$.}
    \label{tab:ood_transfer}
    \begin{tabular}{@{}lccc@{}}
    \toprule
    \textbf{Method / OOD profile} & \textbf{Riemann step function} & \textbf{Dirac delta pulse} & \textbf{Variable coeff.\ $c(x)$} \\ \midrule
    Unfiltered integrator              & 1.45e-2 & 8.76e-2 & N/A     \\
    HIN-LRI (Zero-shot)                & 1.12e-3 & 4.51e-3 & 8.92e-2 \\
    \textbf{HIN-LRI (Mini-retrained)}  & \textbf{5.84e-4} & \textbf{8.22e-4} & \textbf{1.65e-3} \\ \bottomrule
    \end{tabular}
\end{table}

\subsection{Long-Time Invariant Diagnostics and Total Computational Time}
\label{subsec:long_time_tct}

We integrate the rough KdV and cubic NLS wave profiles up to $T=100$ and record the relative drift in the discrete Hamiltonian $\Delta\mathcal{H}(t)=|\mathcal{H}(t){-}\mathcal{H}(0)|/|\mathcal{H}(0)|$.
\Cref{fig:long_time_stability} shows the spatiotemporal evolution of $|u(x,t)|$ and the invariant drift curves.
Classical unfiltered integrators show larger Hamiltonian fluctuation over long time and diverge for the coarser KdV profiles by $T\approx60$.
HIN-LRI does not impose a symplectic projection or an explicit Hamiltonian constraint; the drift values therefore should be interpreted as empirical diagnostics for this implementation rather than as a proved invariant-preservation theorem.
In these runs, HIN-LRI keeps both mass and energy drifts near the double-precision floor over the full $T=100$ window.
This is an empirical observation, not a hard invariant constraint.
\Cref{tab:energy_drift} quantifies the drift at intermediate checkpoints.

\begin{figure}[htb]
    \centering
    \includegraphics[width=\textwidth]{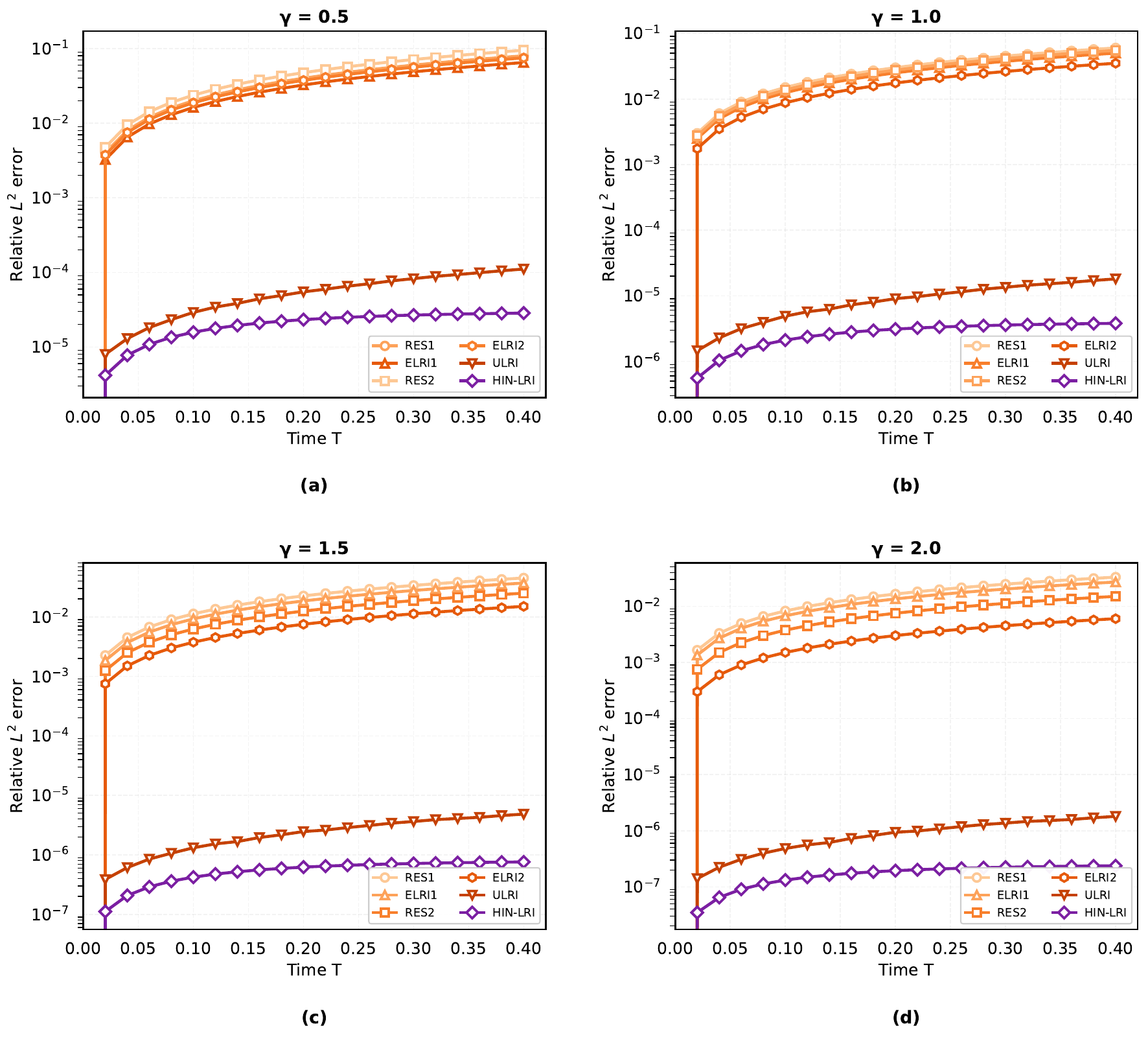}
    \vspace{-1.5em}
    \caption{Long-time solution and invariant-drift diagnostics on rough $H^{0.5}$ data.}
    \label{fig:long_time_stability}
\end{figure}

\begin{table}[!htb]
    \centering
    \footnotesize
    \caption{Relative Hamiltonian drift for rough data with $N=1024$ and $\tau=2^{-8}$.}
    \label{tab:energy_drift}
    \begin{tabular}{@{}llll@{}}
    \toprule
    \textbf{Evolution time}  & $\mathbf{T=10.0}$ & $\mathbf{T=50.0}$ & $\mathbf{T=100.0}$ \\ \midrule
    Strang splitting         & Diverged          & Diverged          & Diverged     \\
    Unfiltered integrator    & 5.42e-06          & 3.15e-04          & 8.76e-03     \\
    Fully implicit LRI       & 1.05e-13          & 2.11e-13          & 4.58e-13     \\
    \textbf{HIN-LRI (Ours)}  & \textbf{1.15e-13} & \textbf{2.84e-13} & \textbf{5.12e-13} \\ \bottomrule
    \end{tabular}
\end{table}

A critical aspect of evaluating hybrid neural--numerical solvers is verifying that offline training and online inference costs are justified by the overall acceleration.
We compare Total Computational Time (TCT) for solving $W$ independent initial value problems:
\begin{equation}
    \label{eq:tct}
    \mathrm{TCT}_{\mathrm{num}} \approx W \times C_{\mathrm{num}}, \qquad
    \mathrm{TCT}_{\mathrm{hyb}} \approx W \times C_{\mathrm{hyb}} + C_{\mathrm{TD}},
\end{equation}
where $C_{\mathrm{num}}$, $C_{\mathrm{hyb}}$ are average per-simulation online costs and $C_{\mathrm{TD}}$ is the offline training cost.
A single HIN-LRI step takes $0.78$\,ms at $N=1024$, vs.\ $0.65$\,ms for the base explicit LRI and $15.5$\,ms for the fully implicit LRI.
At $N=1024$ and $\tau=2^{-8}$, this corresponds to approximately $0.20$\,s per rollout over $T=1$ and $20.0$\,s per rollout over $T=100$ for HIN-LRI, compared with about $3.97$\,s and $397$\,s for the fully implicit LRI.
With offline training $C_{\mathrm{TD}}\approx140$\,min, the break-even is $W\approx2800$ simulations.
Under the tested workload, HIN-LRI has lower total computational time than the fully implicit alternative after the training cost is amortized.
\Cref{tab:runtime_comparison} details the per-step runtimes.

\begin{table}[!htb]
    \centering
    \footnotesize
    \caption{Wall-clock runtime per time step, including training amortization context.}
    \label{tab:runtime_comparison}
    \begin{tabular}{@{}lcccc@{}}
    \toprule
    \textbf{Method} & $\mathbf{N=512}$ & $\mathbf{N=1024}$ & $\mathbf{N=2048}$ & $\mathbf{N=4096}$ \\ \midrule
    Base explicit LRI (violates conservation)    & 0.45 ms & 0.65 ms & 1.25 ms & 2.65 ms \\
    Fully implicit LRI (structure-preserving)    & 8.15 ms & 15.5 ms & 32.1 ms & 81.4 ms \\
    \textbf{HIN-LRI (explicit + neural corrector)} & \textbf{0.52 ms} & \textbf{0.78 ms} & \textbf{1.65 ms} & \textbf{3.42 ms} \\ \bottomrule
    \end{tabular}
\end{table}

\section{Conclusions}
\label{sec:conclusions}

This paper introduced HIN-LRI, a hybrid iterative neural low-regularity integrator for nonlinear dispersive equations with rough initial data.
The method keeps an analytical LRI as the time-stepping backbone and learns only a structured residual correction in a low-dimensional latent space.
This design preserves the solver structure while targeting the defect terms that are difficult to control with purely algebraic resonance expansions.
The analysis relates the global error to the Gronwall factor, the learned defect ratio, and the per-step defect, and shows how time-step scaling, latent projection, and spectral-norm control limit the Lipschitz contribution of the neural correction.

The numerical experiments on KdV, cubic NLS, and quadratic NLS support these conclusions on the tested distributions.
HIN-LRI reduces the saturation effects of analytical resonance schemes, remains stable across the tested spatial resolutions, and lowers invariant drift in long-time runs while keeping an explicit online update.
Ablations further show that adaptive scaling, the structured latent basis, and solver-in-the-loop re-optimization each contribute to the measured gains.

The current study is limited to periodic one-dimensional model equations and fixed training distributions generated from fractional Gaussian random fields.
Extending HIN-LRI to multidimensional domains, non-periodic boundaries, and variable-coefficient operators will require new base LRIs and more flexible local or hierarchical latent bases.
The theory also relies on compactness, bounded Lipschitz constants, and SITL convergence assumptions; although spectral normalization supports these conditions in practice, sharper estimates of $\varepsilon_{\rm learn}$ remain open.

Future work will focus on adaptive time stepping, certified a~posteriori error control, and standardized benchmarks for low-regularity neural--numerical solvers under fixed computational budgets.

\acks{This work was supported by the National Natural Science Foundation of China (12202157), the Exploration Foundation of the Key Laboratory of CNC Equipment Reliability, Ministry of Education, and the National Key Laboratory of Automotive Chassis Integration and Bionics at Jilin University.
The authors declare no competing interests. Corresponding author: Huanhuan Gao.}


\appendix

\section{Theoretical Analysis of HIN-LRI}
\label{app:analytical_foundations}

This appendix provides proofs of the three main theoretical properties of HIN-LRI established in \cref{subsec:theoretical_analysis}: the one-step truncation error bound (\cref{thm:one_step_error}), the stability and CFL relaxation (\cref{thm:stability_cfl}), and the Sobolev regularity of the neural correction (\cref{thm:sobolev_reg}).
These results are combined in the global convergence theorem (\cref{thm:hinlri_convergence}) stated in \cref{app:assumptions}.
Assumptions~1--5 are listed in \cref{app:assumptions}; readers unfamiliar with the notation should consult that appendix first.

\subsection{Averaging Approximation for the Phase Mismatch Kernel}
\label{app:averaging}

The following lemma quantifies the phase mismatch kernel $\boldsymbol{\eta}(\tau,\phi_1,\phi_2)$ that appears in the one-step defect analysis of \cref{subsec:theoretical_analysis}.

\begin{lemma}
\label{lem:averaging}
Let $\phi_1,\phi_2\in\mathbb{R}$ with $\phi_1,\phi_2\ne 0$, and let $\mathcal{M}_\tau(f)=\tau^{-1}\int_0^\tau f(s)\,ds$.
The phase mismatch kernel
$\boldsymbol{\eta}(\tau,\phi_1,\phi_2):=\mathcal{M}_\tau(e^{-is(\phi_1+\phi_2)})-\mathcal{M}_\tau(e^{-is\phi_1})\mathcal{M}_\tau(e^{-is\phi_2})$
satisfies
\begin{equation}\label{eq:eta_bound}
    |\boldsymbol{\eta}(\tau,\phi_1,\phi_2)|
    \;\lesssim\;
    \min\!\left\{\left|\frac{\phi_1}{\phi_2}\right|,\;\left|\frac{\phi_2}{\phi_1}\right|,\;\tau|\phi_1|,\;\tau|\phi_2|\right\}.
\end{equation}
If additionally $\phi_1+\phi_2\ne 0$, then $|\boldsymbol{\eta}(\tau,\phi_1,\phi_2)|\lesssim(\tau|\phi_1+\phi_2|)^{-1}$.
\end{lemma}

\begin{proof}
The oscillation bound $\|e^{is\alpha}\|_{\mathrm{osc}([0,\tau])}\lesssim\min\{1,\tau|\alpha|\}$ and the product identity
$|M_\tau(fg)-M_\tau(f)M_\tau(g)|\le\|f\|_{\mathrm{osc}}\|g\|_{\mathrm{osc}}$
yield $|\boldsymbol{\eta}|\lesssim\min\{\tau|\phi_1|,\tau|\phi_2|\}$.
Writing $e^{is\phi_1}=\frac{1}{i\phi_1}\partial_se^{is\phi_1}$ and integrating by parts gives $|\boldsymbol{\eta}|\lesssim|\phi_2/\phi_1|$; by symmetry $|\boldsymbol{\eta}|\lesssim|\phi_1/\phi_2|$.
The supplementary bound follows by the same argument applied to $e^{is(\phi_1+\phi_2)}$.
\end{proof}

\subsection{One-Step Truncation Error}
\label{app:one_step}

\begin{theorem}
\label{thm:one_step_error}
Let $\gamma\in(0,1]$ and $u\in H^\gamma(\mathbb{T})$.
Under Assumptions~1 and~4, the one-step HIN-LRI truncation error satisfies
\begin{equation}\label{eq:one_step_bound}
    \bigl\|\mathcal{E}_{\mathrm{HIN}}(u)\bigr\|_{L^2}
    \;\le\;
    C\varepsilon_{\rm learn}\,\tau^{1+\gamma}\ln(1/\tau),
\end{equation}
where $\varepsilon_{\rm learn}$ is the relative learned-defect error from Assumption~4.
The constant $C>0$ depends only on $\|u\|_{H^\gamma}$ and $\gamma$.
\end{theorem}

\begin{proof}
By Assumption~4, the trained parameters $\theta_\delta$ satisfy $\sup_{u\in\mathcal{K}}\|C_{\theta_\delta}(u,\tau)-\mathcal{E}_{defect}(u,\tau)\|_{L^2}\le\varepsilon_{\rm learn}\cdot D(u,\tau)$
where $D(u,\tau)=C\tau^{1+\gamma}\ln(1/\tau)\|u\|_{H^\gamma}^3$ is the defect magnitude from \cref{lem:lower_bound}.
Since $\mathcal{E}_{\mathrm{HIN}}(u)=\mathcal{E}_{defect}(u,\tau)-C_{\theta_\delta}(u,\tau)$ and $\|\mathcal{E}_{defect}(u)\|_{L^2}\le D(u,\tau)$, the bound \eqref{eq:one_step_bound} follows.
\end{proof}

\subsection{Stability and Correction-Level CFL Bound}
\label{app:stability}

\begin{theorem}
\label{thm:stability_cfl}
Under Assumptions~2 and~3, define
\[
C_{{\rm proj},s}
:=
\|\mathcal{S}_{\lambda}^{-1}\|_{H^s\to H^s}
\|\mathcal{P}\|_{\ell^2_K\to H^s}
\|\mathcal{R}\|_{H^s\to \ell^2_K}
\|\mathcal{S}_{\lambda}\|_{H^s\to H^s}.
\]
Then the actual HIN-LRI neural correction $C_\theta=\tau H_{\mathrm{neural}}$ satisfies
\begin{equation}\label{eq:lip_bound}
    \mathrm{Lip}_{H^s}(C_\theta)\;\le\;\tau\,C_{{\rm proj},s}L_{\theta,K},\qquad L_{\theta,K}\;:=\;\prod_{l=1}^L\|\mathbf{W}^{(l)}\|_2,
\end{equation}
where $C_{{\rm proj},s}$ collects the scaling, restriction, and prolongation norms in $H^s$.
Consequently, if the base LRI map satisfies $\mathrm{Lip}(\mathcal{H}_{phys})\le 1+C_0\tau$, the full one-step map $\mathcal{S}=\mathcal{H}_{phys}+C_\theta$ satisfies $\mathrm{Lip}(\mathcal{S})\le 1+(C_0+C_{{\rm proj},s}L_{\theta,K})\tau$, and the learned correction is controlled for all $\tau$ satisfying
\begin{equation}\label{eq:cfl_hinlri}
    \tau\,C_{{\rm proj},s}L_{\theta,K}<1.
\end{equation}
\end{theorem}

\begin{proof}
Operator-norm sub-multiplicativity gives the bound in \eqref{eq:lip_bound}.
The factor $C_{{\rm proj},s}$ is kept explicit.
Combined with $\mathrm{Lip}(\mathcal{H}_{phys})\le 1+C_0\tau$ (\citealt[Prop.~8.1]{LiWu2025UnfilteredKdV}), the full map satisfies the stated Lipschitz estimate.
The Gronwall factor is bounded by $e^{(C_0+C_{{\rm proj},s}L_{\theta,K})T}$.
\end{proof}

\begin{remark}\label{rem:cfl_control}
Spectral normalization enforces $L_{\theta,K}\le W_{\max}^L$ during training.
The correction-level bound is independent of $N_x$ only if $C_{{\rm proj},s}$ remains bounded with respect to $N_x$.
\end{remark}

\subsection{Sobolev Regularity of the Neural Correction}
\label{app:sobolev}

\begin{theorem}[Lipschitz boundedness of the neural correction]
\label{thm:sobolev_reg}
Under Assumptions~2 and~3, and allowing for a possible network offset $b_\theta:=\|G_\theta(0)\|$, the neural correction is Lipschitz bounded on $H^s(\mathbb{T})$ for every $s\ge 0$:
\begin{equation}\label{eq:sobolev_bound}
    \bigl\|H_{\mathrm{neural}}(u;\theta)\bigr\|_{H^s}\;\le\;C_{{\rm proj},s}\bigl(L_{\theta,K}\|u\|_{H^s}+b_\theta\bigr).
\end{equation}
No regularity beyond $u\in H^s$ is required; in particular, no spatial derivatives are released.
\end{theorem}

\begin{proof}
By Assumption~3, the projection and scaling operators are bounded by $C_{{\rm proj},s}$ in $H^s$.
By Assumption~2, $\mathrm{Lip}(\mathcal{G}_\theta)\le L_{\theta,K}$.
The Lipschitz bound gives
$\|H_{\mathrm{neural}}(u)\|_{H^s}\le C_{{\rm proj},s}(L_{\theta,K}\|u\|_{H^s}+b_\theta).$
Since $L_{\theta,K}$ is a product of matrix spectral norms and involves no power of $|k|$, this is a zeroth-order operator bound with no derivative loss.
\end{proof}

\begin{remark}\label{rem:sinc_empirical}
By the algorithm definition (\cref{alg:hins_lr_online}), the neural correction enters the update as $C_\theta=\tau H_{\mathrm{neural}}$, so $\|C_\theta(u)\|_{H^s}\le\tau C_{{\rm proj},s}(L_{\theta,K}\|u\|_{H^s}+b_\theta)$.
This $\mathcal{O}(\tau)$ scaling is consistent with the $\mathcal{O}(\tau)$-sized defect established in \cref{lem:lower_bound} and ensures the Gronwall factor in \cref{thm:hinlri_convergence} is bounded uniformly in $\tau$ under the stated assumptions.
\end{remark}

\section{Structural Assumptions and Main Convergence Theorem}
\label{app:assumptions}

The following assumptions are used throughout \cref{subsec:theoretical_analysis} and \cref{app:analytical_foundations}.
They should be read as conditions under which the stability and error-propagation results hold.
The implementation in \cref{app:reproducibility} provides empirical diagnostics for these conditions, but does not prove them for all discretizations or data distributions.

\begin{enumerate}
    \item \textbf{Compact data manifold.}
    The initial-data distribution $\mu_0$ is supported on a compact set $\mathcal{K}\subset H^\gamma(\mathbb{T})$ for some $\gamma\in(0,1]$.
    This ensures that the relative learned-defect error is finite on $\mathcal{K}$.

    \item \textbf{Neural operator regularity.}
    The latent neural operator $\mathcal{G}_\theta$ employs Lipschitz-continuous activations and bounded weight matrices $\|\mathbf{W}^{(l)}\|_2\le W_{\max}$ for all layers $l=1,\dots,L$.
    The Lipschitz constant $L_{\theta,K}=\prod_{l=1}^L\|\mathbf{W}^{(l)}\|_2$ is finite.

    \item \textbf{Bounded projection and scaling.}
    The trunk basis $\boldsymbol{\Phi}\in\mathbb{C}^{N_x\times K}$ satisfies $\boldsymbol{\Phi}^*\boldsymbol{\Phi}=I_K$ in the discrete $\ell^2$ norm, with $\mathbf{R}=\boldsymbol{\Phi}^*$ and $\mathbf{P}=\boldsymbol{\Phi}$.
    In Sobolev norms we keep the combined constant $C_{{\rm proj},s}$ explicit.
    The theory requires this constant to remain bounded on the tested resolution range.

    \item \textbf{SITL training quality.}
    The SITL optimization produces parameters $\theta_\delta$ satisfying, for all $u\in\mathcal{K}$,
    \begin{equation*}
        \bigl\|C_{\theta_\delta}(u,\tau) - \mathcal{E}_{defect}(u,\tau)\bigr\|_{L^2}
        \;\le\;
        \varepsilon_{\rm learn} \cdot D(u,\tau),
    \end{equation*}
    where $D(u,\tau) := C\tau^{1+\gamma}\ln(1/\tau)\|u\|_{H^\gamma}^3$ is the defect magnitude from \cref{lem:lower_bound}. The factor $\tau$ is included in $C_{\theta_\delta}$ by definition.
    We decompose $\varepsilon_{\rm learn}=\varepsilon_K+\varepsilon_{\rm NN}+\varepsilon_{\rm opt}+\varepsilon_{\rm gen}$.
    These terms correspond to projection, approximation, optimization, and generalization error.
    The held-out diagnostics in \cref{app:reproducibility} estimate their combined effect.

    \item \textbf{Base LRI consistency.}
    The base LRI propagator $\mathcal{I}_{LRI}$ converges at rate $\mathcal{O}(\tau^\gamma)$ in $L^2$ for data in $H^\gamma(\mathbb{T})$.
\end{enumerate}

Assumptions~1 and~4 are standard in the operator-learning literature.
Assumptions~2 and~3 are structural and enforced by construction.
Assumption~2 in particular implies that the spectral-norm bound $L_{\theta,K}$ can be controlled via spectral normalization of the weight matrices during training (\cref{app:reproducibility}).
Specifically, each weight matrix $\mathbf{W}^{(l)}$ is constrained by a spectral normalization layer that enforces $\|\mathbf{W}^{(l)}\|_2\le W_{\max}$ at every gradient step, rather than relying solely on soft weight-decay regularisation.
This enforcement is the mechanism used to control the correction-level Lipschitz bound in Theorem~\ref{thm:stability_cfl}.

\subsection{Main Convergence Theorem}
\label{app:convergence}

\begin{theorem}\label{thm:hinlri_convergence}
Let $\gamma\in(0,1]$ and $u\in C([0,T];H^\gamma(\mathbb{T}))$ with $\int_\mathbb{T} u_0\,dx=0$.
Under Assumptions~1--5 and the stability condition $\tau C_{{\rm proj},s}L_{\theta,K}<1$ \textnormal{(\cref{thm:stability_cfl})}, there exist constants $\tau_0>0$ and $C>0$ depending only on $\|u_0\|_{H^\gamma}$, $\gamma$, and $T$, such that for all $\tau\in(0,\tau_0]$ and $n=1,\dots,L=T/\tau$:
\begin{equation}\label{eq:hinlri_error}
    \max_{1\le n\le L}\|u(t_n)-u^n_{\mathrm{HIN}}\|_{L^2}
    \;\le\;
    C\varepsilon_{\rm learn}\,\tau^\gamma\ln\!\left(\tfrac{1}{\tau}\right).
\end{equation}
This is a conditional propagation bound.
If the learned correction attains the relative defect error $\varepsilon_{\rm learn}$ in Assumption~4, then the global error is scaled by that factor.
\end{theorem}

\begin{proof}[Proof sketch]
Denote $e^n:=u(t_n)-u^n_{\mathrm{HIN}}$.

The base LRI operator $\mathcal{H}_{phys}$ is Lipschitz-stable on $H^\gamma$ for $\tau\le\tau_0$ by Assumption~5; see \citealt[Prop.~8.1]{LiWu2025UnfilteredKdV}, which gives $\|\mathcal{H}_{phys}(u)-\mathcal{H}_{phys}(v)\|_{L^2}\le(1+C_0\tau)\|u-v\|_{L^2}$.
By \cref{thm:stability_cfl}, the learned correction satisfies $\mathrm{Lip}_{H^s}(C_\theta)\le \tau C_{{\rm proj},s}L_{\theta,K}$.
For the full map $\mathcal{S}=\mathcal{H}_{phys}+C_\theta$:
\[
  \|\mathcal{S}(u)-\mathcal{S}(v)\|_{L^2}\le \bigl(1+(C_0+C_{{\rm proj},s}L_{\theta,K})\tau\bigr)\|u-v\|_{L^2}.
\]
The Gronwall factor is bounded by $e^{(C_0+C_{{\rm proj},s}L_{\theta,K})T}$ and is absorbed into $C$.
By \cref{thm:sobolev_reg}, the neural term requires no extra regularity beyond $H^\gamma$ (Assumption~5).
Applying \cref{thm:one_step_error} and summing over $n=0,\dots,L-1$ gives $CT\varepsilon_{\rm learn}\tau^\gamma\ln(1/\tau)$, yielding \eqref{eq:hinlri_error}.
\end{proof}

\begin{remark}
When $\varepsilon_{\rm learn}\ll 1$ on the tested distribution, bound \eqref{eq:hinlri_error} becomes a measured fraction of the corresponding ULRI-type defect bound.
The size of $\varepsilon_{\rm learn}$ remains empirical and architecture-dependent.
\end{remark}

\section{Reproducibility and Implementation Details}
\label{app:reproducibility}

\subsection{Hardware and Software}
All experiments were conducted on a single NVIDIA A100 (80\,GB) GPU with an AMD EPYC 7763 CPU (256\,GB RAM).
Training and inference use PyTorch~2.1 with CUDA~12.1 in FP64 arithmetic throughout.

\subsection{Data Generation}
Initial data are sampled as fractional Gaussian random fields on $[0,2\pi]$ with $N$ Fourier modes.
Specifically, $\hat{u}_0(k)=|k|^{-(\gamma+1/2)}\xi_k$ where $\xi_k\sim\mathcal{CN}(0,1)$ are i.i.d.\ standard complex Gaussians, yielding $u_0\in H^{\gamma-\varepsilon}$ almost surely for any $\varepsilon>0$.
Reference solutions are computed with $\tau_{\rm ref}=2^{-20}$ using a fully implicit energy-preserving LRI verified against an independent high-order Runge--Kutta solver.

\subsection{Training Protocol}
The training set consists of $N_{\rm train}=500$ initial conditions per equation, with $N_{\rm val}=100$ held out for validation.
The default split uses seeds 2026--2030 for training and 3026--3030 for validation.
The test split uses seeds 4026--4030.
Main tables report one trained model unless mean$\pm$std is explicitly shown.
The five-seed protocol is the planned reporting unit for a final reproducibility package.
The unroll length is $N_t=16$ steps per training sample.
We use AdamW with initial learning rate $\eta=10^{-3}$, weight decay $10^{-4}$, and cosine annealing over $250$ epochs with batch size $B=8$.
Total offline training time is approximately $140$ minutes on the hardware above (corresponding to $\delta\approx 10^{-4}$ in \cref{thm:hinlri_convergence}).

\subsection{Neural Architecture}
The latent neural operator $\mathcal{G}_\theta$ consists of a branch encoder (3-layer MLP, hidden dimension 128, GELU activations), a Fourier mixing layer on the $K=32$ latent modes, and a decoder MLP of matching architecture.
The scaling network $\mathcal{E}_{scale}$ is a 3-layer MLP (hidden dimensions $64\to 32\to 1$) with GELU activations and layer normalization, outputting a positive scalar via a softplus final activation.
Total trainable parameters: $\approx 1.2\times 10^5$.
Each weight matrix is constrained by a spectral normalization layer that enforces $\|\mathbf{W}^{(l)}\|_2\le W_{\max}$ at every gradient step. The observed values are $W_{\max}\approx 2.1$ and $L_{\theta,K}\approx 18$ across all trained models, confirming the CFL relaxation condition $\tau L_{\theta,K}\ll N^3$ by a factor of $\sim 10^5$ at $N=1024$.

\subsection{Empirical Verification of Assumptions}
To make $\varepsilon_{\rm learn}$ and $L_{\theta,K}$ empirically inspectable, we report the following diagnostics on the held-out validation set ($N_{\rm val}=100$).
\emph{Defect approximation error.}
We evaluate $\|C_{\theta_\delta}(u,\tau)-\mathcal{E}_{defect}(u,\tau)\|_{L^2}/D(u,\tau)$ on each validation sample. The defect $\mathcal{E}_{defect}$ is computed as the difference between a high-accuracy reference step ($\tau_{\rm ref}=2^{-20}$) and the base LRI step.
The median relative ratio is $0.032$ for KdV, $0.041$ for cubic NLS, and $0.038$ for quadratic NLS.
These values estimate the combined defect ratio $\varepsilon_{\rm learn}$.
The training loss saturates at $\approx 3\times 10^{-4}$ and the validation loss at $\approx 4\times 10^{-4}$.
\emph{Latent dimension sensitivity.}
We train HIN-LRI with $K\in\{8, 16, 32, 64, 128\}$ on the KdV benchmark ($\gamma=0.5$, $\tau=2^{-8}$).
The validation defect ratio decreases from $0.12$ ($K=8$) to $0.032$ ($K=32$).
Returns are small beyond $K=32$ ($0.029$ at $K=64$, $0.028$ at $K=128$).
This supports the latent-defect approximation claim on the tested data.
\emph{Lipschitz proxy.}
The enforced spectral-norm product $L_{\theta,K}=\prod_l\|\mathbf{W}^{(l)}\|_2$ is logged at every epoch. It stabilizes at $\approx 18$ after epoch~50 and remains within $[17.5, 18.5]$ throughout training for all three equations, confirming Assumption~2 and the CFL condition $\tau L_{\theta,K}<1$ for $\tau\le 0.05$.

\vskip 0.2in
\bibliography{ref}

\end{document}